\documentclass[letterpaper, 10 pt, conference]{ieeeconf}  % Comment this line out
\IEEEoverridecommandlockouts                              % This command is only
\usepackage[utf8]{inputenc}
\usepackage[T1]{fontenc}
\usepackage{cite}
\usepackage{amsmath,amssymb,amsfonts}
\usepackage{algorithmic}
\usepackage{graphicx}
\usepackage{textcomp}
\usepackage{xcolor}
\def\BibTeX{{\rm B\kern-.05em{\sc i\kern-.025em b}\kern-.08em
    T\kern-.1667em\lower.7ex\hbox{E}\kern-.125emX}}
    
\usepackage{tikz}
\usepackage{tikzscale}
\usepackage{tkz-euclide}
\usepackage{subcaption}
\usepackage{caption}
\usepackage{pgfplots}
\usepackage{url}
\usepackage{tabularx,booktabs}
  
\usetikzlibrary{shapes}
\usetikzlibrary{matrix,arrows,fit,backgrounds,mindmap,plotmarks,decorations.pathreplacing}
\usetikzlibrary{shapes.multipart}
\usetikzlibrary{intersections}
\usepgfplotslibrary{colormaps}
\newtheorem{theorem}{Theorem}
\definecolor{thupurple}{RGB}{0,0,0}

\newlength{\spacelength}
\title{\LARGE \bf
A Hierarchical Control Framework for Drift Maneuvering of Autonomous Vehicles
}

\author{Bo Yang, Yiwen Lu, Xu Yang and Yilin Mo$^{1}$% <-this % stops a space
\thanks{*This work is supported by the National Key Research and Development Program of China under Grant 2018AAA0101601}% <-this % stops a space
\thanks{$^{1}$The authors are with Department of Automation and BNRist, Tsinghua University, Beijing, China
        {\tt\small \{yang-b21,luyw20,yangx21\}@mails.tsinghua.edu.cn, ylmo@tsinghua.edu.cn}}%
}

\begin{document}

\maketitle
\thispagestyle{empty}
\pagestyle{empty}

%%%%%%%%%%%%%%%%%%%%%%%%%%%%%%%%%%%%%%%%%%%%%%%%%%%%%%%%%%%%%%%%%%%%%%%%%%%%%%%%
\begin{abstract}

Maneuvering an autonomous vehicle under drift condition is critical to the safety of autonomous vehicles when there is a sudden loss of traction due to external conditions such as rain or snow, which is a challenging control problem due to the presence of significant sideslip and nearly full saturation of the tires.
In this paper, we focus on the control of drift maneuvers of autonomous vehicle to track circular paths with either fixed or moving centers, subject to change in the tire-ground interaction, which are common training tasks for Radio Control~(RC) car drifting enthusiasts and can therefore be used as benchmarks of the performance of drift control.
In order to achieve the above tasks, we propose a hierarchical control architecture which decouples the curvature and center control of the trajectory.
In particular, an outer control loop is proposed to stabilize the center by tuning the target curvature, and an inner control loop tracks the curvature using a feedforward/feedback controller enhanced by an $\mathcal{L}_1$ adaptive component.
The hierarchical architecture is flexible because the inner loop is task-agnostic and adaptive to changes in tire-ground interaction, which allows the outer loop to be designed independent of low-level dynamics, opening up the possibility of incorporating sophisticated planning algorithms.
We implement our control strategy on a simulation platform as well as on a 1/10 scale RC car, and both the simulation and experiment results illustrate the effectiveness of our strategy in achieving the above described set of drift maneuvering tasks.

\end{abstract}

%%%%%%%%%%%%%%%%%%%%%%%%%%%%%%%%%%%%%%%%%%%%%%%%%%%%%%%%%%%%%%%%%%%%%%%%%%%%%%%%
\section{Introduction}

In recent years, increasing attention from academia and industry has been drawn to autonomous driving research~\cite{paden2016survey}, which includes various sub-areas such as environment perception~\cite{perception1,perception2,perception3}, motion planning~\cite{planning1,planning2,planning3}, motion control~\cite{control1,control2,control3} and etc. In this paper, we focus on drift control, an extreme instance of motion control where significant slip exists between the tires and the ground. Drift control is pertinent to the safety of autonomous driving because significant slip may occur unexpectedly due to external conditions like rain or snow and cause a sudden loss of traction, where the precise control of the vehicle trajectory is required to avoid accidents. More specifically, we focus on \emph{drift maneuvering}, which involves manipulating the vehicle to track certain trajectories in sustained drift, and can therefore be viewed as a benchmark of drift control performance.

Drift maneuvering is challenging due to tire saturation and limited control authority in a highly unstable region~\cite{zhang2018drift}. As a result, in the existing literature it is usually tackled on a per-task or per-vehicle basis in a controlled environment. For example, Goh et al.~\cite{goh2016simultaneous} reformulate a fixed circular trajectory as a series of vehicle-dependent drift equilibrium points, based on which they propose tracking controllers to minimize the sideslip error and lookahead error, and Goh et al.~\cite{goh2020toward} extend the above work to tracking arbitrary pre-defined trajectories by casting the vehicle dynamics into a trajectory-dependent curvilinear coordinate system, in which they apply nonlinear model inversion to minimize the tracking error. Both works rely heavily on accurate modeling of the vehicle dynamics and parameterization of the reference trajectory, which is very demanding if the methods are to be deployed to different vehicles, or even the same vehicle whose model inevitably varies due to load changes or component wears, and applied to perform generic drift maneuvering tasks. Culter et al.~\cite{cutler2016autonomous}, on the contrary, adopt the data-driven approach by applying an reinforcement learning algorithm called Probability Inference for Learning COntrol~(PILCO), but the method is only considered in a single-task setting, where the objective is to minimize tracking error of a particular drift equilibrium. Like other reinforcement learning algorithms~\cite{dulac2021challenges}, ad-hoc reward shaping and a significant amount of training data may be required for the method to generalize across different drift maneuvering tasks. Apart from the model-based and data-driven methods, another class of drift maneuvering methods are aided by expert experience, including Jelavic et al.~\cite{Jelavic2017AutonomousDP} and Zhang et al.~\cite{ZHANG20171916,zhang2018drift}, which apply either a feedforward/feedback controller or an open-loop controller to track an expert trajectory in some phase of the drifting process. A drawback of the expert-aided methods is that human experience has limited coverage of different drifting conditions, and as a result, those methods are usually designed to track a particular trajectory. Moreover, none of the above works consider the sudden change of tire-ground interaction, a practical situation that may occur due to external conditions like rain or snow.

In this paper, we propose a drift control framework that can perform various maneuvering tasks and adapt to sudden changes of tire-ground interaction. The flexibility of our framework stems from the observation that a generic drift trajectory can be viewed as an arc with continuously varying center and curvature, based on which we propose a hierarchical architecture that decouples the center and curvature control of the trajectory. In particular, an outer loop stabilizes the center by tuning the target curvature, and an inner loop tracks the curvature using a feedforward/feedback controller enhanced by an $\mathcal{L}_1$ adaptive control~\cite{michini2009l1} component. Through an adaptive feedforward and reference signal design, the inner loop can deliver a consistent curvature tracking performance over different tire-ground interaction, allowing the outer loop to focus on maneuver navigation, such that it can achieve various tasks without knowing the internal details of the vehicle dynamics.

A key contribution of our control architecture described above is the introduction of \emph{curvature inference and feedback}, which bridge the low-level vehicle stabilization and the high-level navigation. In contrast to the aforementioned previous works that are built upon state feedback, our proposed controller makes decisions based on both the vehicle state and the trajectory curvature. The reason behind this design are twofold: On one hand, as a quantity derived from historical trajectory in a small time window, curvature can encodes higher-level information about the drifting state comparing to other quantities computed directly from the current state, e.g. sideslip angle. Hence, observed from the experiments on a RC car platform, the inner control loop can be designed with better performance. On the other hand, the curvature information can be leveraged by the outer control loop for planning desired trajectories, which in this paper include circular motions orbiting fixed or moving centers.

In this paper, we illustrate the effectiveness of our control strategy on the circular drift maneuvers, subject to changes in the tire-ground interaction, which are common training tasks for drifting enthusiasts~\cite{abdulrahim2006dynamics}. We also believe our proposed hierarchical control architecture can potentially incorporate more sophisticated planning algorithms, e.g., the one proposed by Levin et al.~\cite{levin2019agile}, which plans trajectories by concatenating motion primitives. Conceptually complex trajectory can be approximated in terms of a sequence of arcs with fixed curvature values (motion primitives), and such sequences, common in the aforementioned drift training tasks, have been shown to be accurately tracked by our proposed controller in both simulation and hardware experiments. 

The rest of the manuscript is organized as follows: Section~\ref{sec:model} introduces the modeling of our RC car platform and defines drift maneuvering tasks, Section~\ref{sec:method} gives an overview of our proposed control architecture and describes its components in detail, Section~\ref{sec:result} presents the result of our control strategy on a simulation platform as well as on our RC car, and finally, Section~\ref{sec:conclusion} summarizes this paper and gives remarks on future research directions.

\vspace{\spacelength}
\section{Problem Formulation}\label{sec:model}

The target platform of our control design is a 1/10 scale RC car~shown in Fig.~\ref{fig:racecar}, which is similar to the ones used in MIT Racecar~\cite{racecar} and BARC~\cite{barc} projects. The vehicle adopts an Ackermann steering geometry, and is four-wheel driven in the sense that all wheels have the same rotational speed.
\begin{figure}[!htbp]
    \centering
    \includegraphics[width=0.618\columnwidth]{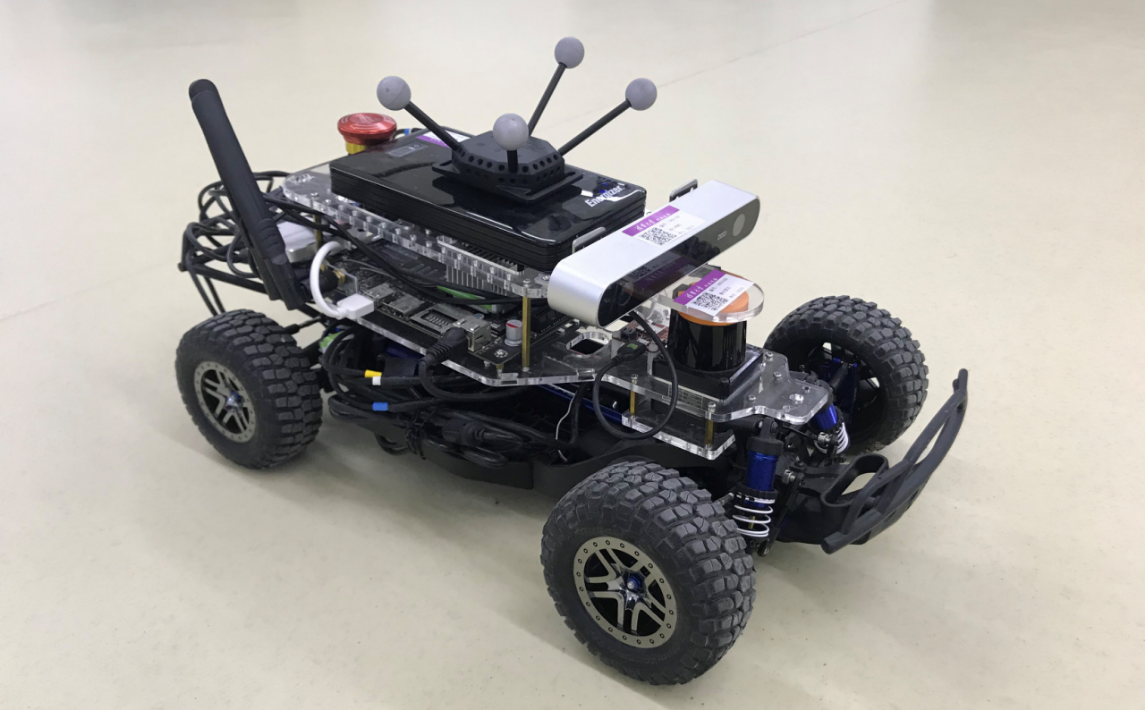}
    \caption{Our RC car platform}
    \label{fig:racecar}
\end{figure}

For controller design and simulation purposes, we abstract the above shown RC car into a bicycle model with sideslip, illustrated in Fig.~\ref{fig:bike_model}. It is worth noticing that we ignore the difference between left and right set of wheels, which is usually insignificant in drift maneuvering~\cite{zhang2018drift}. On the other hand, for a relatively high-fidelity characterization of the tire slip, which is an essential feature of drift, we model the tire-ground interaction using the wide-adopted Pacejka Magic Formula~\cite{pacejka1997magic}. The above mixed-fidelity vehicle modeling for aggressive maneuvering has been reported and proved effective in~\cite{Jeon2011a,lu2021twotimescale}. 

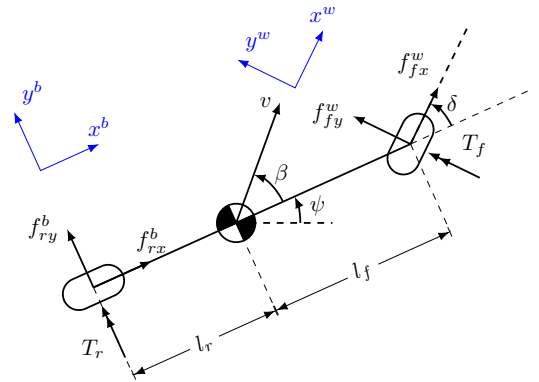
\begin{figure}[!htbp]
    \centering
    \resizebox{\columnwidth}{!}{
        % \hspace{-15cm}
       \usetikzlibrary{calc,angles,quotes,arrows.meta}

\tikzset{pill/.style={minimum width=1cm,minimum height=5mm,rounded corners=2.5mm,draw},
reactor/.style={circle,draw,minimum size=6mm,path picture={
\draw (-3mm,0) -- (3mm,0) (0,-3mm) -- (0,3mm);
\fill (0,0) -- (3mm,0) arc(0:-90:3mm) -- cycle;
\fill (0,0) -- (-3mm,0) arc(180:90:3mm) -- cycle;
}}}

\begin{tikzpicture}
\coordinate (O) (0, 0);
% \draw[blue,{Latex}-{Latex}] (0,-1) node [right] (Yi) {$y^i$} -- (0,-2) coordinate (O)-- (1,-2)
% node [above] (Xi) {$x^i$};
\path (O) -- (9.5,4.3) coordinate[pos=0.28] (F1) coordinate[pos=0.8] (F2) coordinate (TR);
% \draw[thick,dotted] (O |- F2) node[left]{$x$} -- (F2) -- (O -| F2) node[below] {$y$};
\draw[thick] (F1)  -- (F2) node[pos=0.45,sloped,reactor] (M){~}
node[pos=0,sloped,pill]{};
\draw[dashed] (F2) -- (TR);
\draw[thick,-latex] (F1) -- ($(F1)!1cm!0:(F2)$)
node[above]{$f_{rx}^b$};
\draw[thick,-latex] (F1) -- ($(F1)!1cm!90:(F2)$)
node[left]{$f_{ry}^b$};
\draw[thick,dashed] (F2) -- ++ (64:2) coordinate(H)  node[pos=0,sloped,pill,solid]{}
pic ["$\delta$",draw,solid,-latex,angle radius=0.7cm,angle eccentricity=1.3] {angle = TR--F2--H};
\draw[thick,-latex] (F2) -- ($(F2)!1cm!0:(H)$)
node[above left]{$f_{fx}^w$};
\draw[thick,-latex] (F2) -- ($(F2)!1cm!90:(H)$)
node[left]{$f_{fy}^w$};
\path (F2) -- ($(F2)!1.5cm!270:(H)$) coordinate[pos=0.2] (F21) coordinate[pos=0.8] (F22) ;
% \draw[thick,red,-latex] (F2) -- ++ (62:2) coordinate (A2) node{$\alpha_f$}
% pic [draw,solid,black,->,angle radius=1.3cm] {angle = H--F2--A2};;
% \draw[thick,red,-latex] (F1) -- ++ (48:2) coordinate (A1)
% pic ["$\alpha_r$",draw,->,red,thick,angle radius=1cm,angle eccentricity=1.3] {angle = F2--F1--A1};
% \draw[->] let \p1=($(F2)-(F1)$),\n1={-180+atan2(\y1,\x1)},\n2={\n1+180} in 
% ($($(M)!8mm!00:(F1)$)+({cos(\n1+90)*1mm},{sin(\n1+90)*1mm})$) arc(\n1:\n2:8mm)
% node[midway,below,red]{$\omega_z$};
% \draw[latex-] ($(F1)!1mm!-90:(M)$) -- ($(F1)!5mm!-90:(M)$) node[below]{$F_{\gamma r}$};
% \draw[latex-] ($(F2)!1mm!90:(M)$) -- ($(F2)!5mm!90:(M)$) node[below]{$F_{\gamma f}$};
\draw[{Bar}{Latex}-{Latex}{Bar}] ($(F1)!1.5cm!270:(M)$) coordinate (l1) -- 
($(M)!1.5cm!270:(F2)$) coordinate (l2) node[midway,sloped,fill=white]{$l_r$};
\draw[{Bar}{Latex}-{Latex}{Bar}] (l2) -- 
($(F2)!1.5cm!90:(M)$) coordinate (l3) node[midway,sloped,fill=white]{$l_f$};
\draw[dashed] (F1) -- (l1) coordinate[pos=0.2] (F11) coordinate[pos=0.8] (F12) ;
\draw[dashed] (M.center) -- (l2);
\draw[dashed] (F2) -- (l3) ;
\draw[thick,->>,>=latex] (F12) -- (F11) node[midway,below left] {$T_r$};
\draw[thick,->>,>=latex] (F22) -- (F21) node[midway,above right] {$T_f$};

\draw[thick,dashed] (M.center) -- ++(0:1.5) coordinate (M1) node{}
pic ["$\psi$",draw,solid,-latex,angle radius=1cm,angle eccentricity=1.3] {angle = M1--M--F2};

\draw[thick,-latex] (M.center) -- ++(70:2) coordinate (V) node[left]{$v$}
pic ["$\beta$",draw,solid,-latex,angle radius=0.8cm,angle eccentricity=1.3] {angle = F2--M--V};

\draw[blue,{Latex}-{Latex}] ($(F1)!3cm!90:(F2)$) node [above right] {$y^b$} coordinate (Yb) -- ($(F1)!2cm!90:(F2)$) coordinate (Ob) -- ($(Ob)!1cm!-90:(Yb)$) node [above] {$x^b$};
\draw[blue,{Latex}-{Latex}] ($(F2)!3cm!90:(H)$) node [above right] {$y^w$} coordinate (Yw) -- ($(F2)!2cm!90:(H)$) coordinate (Ow) -- ($(Ow)!1cm!-90:(Yw)$) node [above] {$x^w$};

\end{tikzpicture}
    }
    \caption{Illustration of car model}
    \label{fig:bike_model}
\end{figure}

The state and input vectors for the bicycle model are:
\begin{align}
    \mathbf{X}=\left[x, y, \psi, \dot{x}, \dot{y}, \dot{\psi} \right]^\top, \mathbf{U}=\left[\delta, \omega\right]^\top,
    \label{eq:xu}
\end{align}
where $x,y$ are the coordinate of the vehicle in a two-dimensional plane, $\psi$ is the heading angle, $\dot{x}, \dot{y}, \dot{\psi}$ are the time derivatives of $x,y,\psi$; $\delta$ is the front wheel steering angle, and $\omega$ is the rotational speed of the wheels. The system dynamics are governed by 2D rigid body kinetics:
\small
\begin{align}
    m\ddot x &= f_{fx}^w \cos(\psi + \delta) - f_{fy}^w \sin(\psi + \delta) + f_{rx}^b \cos\psi - f_{ry}^b \sin\psi, \label{eq:xddot}
    \\
    m\ddot y &= f_{fx}^w \sin(\psi + \delta) + f_{fy}^w \cos(\psi + \delta) + f_{rx}^b \sin\psi + f_{ry}^b \cos\psi,  \label{eq:yddot}
    \\
    I_z \ddot\psi &= \left( f_{fy}^w\cos\delta + f_{fx}^w \sin\delta \right) l_f - f_{ry}^b l_r, \label{eq:psiddot}
\end{align}
\normalsize
where $m$ is the mass of the vehicle, $I_z$ is the moment of inertia of the vehicle body w.r.t. the $z$ axis,  $r_f, r_r$ are the radii of the wheels, $l_f, l_r$ are the distances of the wheels to the center of mass of the vehicle, and $f_{fx}^w, f_{fy}^w, f_{rx}^b, f_{ry}^b$ are the frictional forces. The subscripts `f', `r' refer to ``front'' and ``rear'' respectively, and the superscripts `w' and `b' refer to ``front wheel frame'' and ``body frame'' respectively\footnote{The distinction of the front wheel and body frame is made due to the existence of the front wheel steering angle}. The frictions are determined by.
\begin{align*}
    f_{fx}^w = \mu_{fx}f_{fz},   f_{fy}^w = \mu_{fy}f_{fz},  %\label{eq:ff} \\  
f_{rx}^b = \mu_{rx}f_{rz},   f_{ry}^b = \mu_{ry}f_{rz},     %\label{eq:fr}
\end{align*}
where $\mu_{fx}, \mu_{fy}, \mu_{rx}, \mu_{ry}$ are the friction coefficients, and $f_{fz}, f_{rz}$ are the normal forces. The friction coefficients are described by
\begin{align}
    \mu_{ij} = - \frac{s_{ij}}{s_{i}} D \sin\left( C \operatorname{atan}(Bs_i) \right) \left(i \in f,r,\, j \in x,y\right),
    \label{eq:mf}
\end{align}
where $B,C,D$ are parameters that vary with the tire-ground interaction properties, and $s_i, s_{ij}$ are slip ratios that can be computed from relative speeds between the wheels and the road:
\begin{align}
    & s_i = \sqrt{s_{ix}^2 + s_{iy}^2}, \quad \left(i \in \{f,r\}\right), \label{eq:def_s} \\
    & s_{fx} = \frac{v_{fx}^w - \omega r_f}{\omega r_f}, s_{fy} = \frac{v_{fy}^w}{\omega r_f}, \label{eq:def_sf}\\
    &
    s_{rx} = \frac{v_{rx}^b - \omega r_r}{\omega r_r}, s_{ry} = \frac{v_{ry}^b}{\omega r_r}, \label{eq:def_sr}
    \\
        & v=\sqrt{\dot{x}^{2}+\dot{y}^{2}},  \beta=\arctan \frac{\dot{y}}{\dot{x}}-\psi, \label{eq:beta}\\
    & v_{f x}^w=v \cos (\beta-\delta)+\dot{\psi} l_{f} \sin \delta,  v_{rx}^b = v\cos\beta, \\
    & v_{f y}^w=v \sin (\beta-\delta)+\dot{\psi} l_{f} \cos \delta,  v_{ry}^b = v\sin\beta - \dot\psi l_r.  \label{eq:vfyw}
\end{align}
Finally, the normal forces $f_{fz}, f_{rz}$ can be determined as:
\begin{align}
        f_{fz} = \frac{l_r - \mu_{rx}h}{l_f + l_r + \left(\mu_{f x} \cos \delta-\mu_{f y} \sin \delta-\mu_{r x}\right)h}mg,\\
        f_{rz} = \frac{l_f + \left( \mu_{f x} \cos \delta-\mu_{f y} \sin \delta \right)h}{l_f + l_r + \left(\mu_{f x} \cos \delta-\mu_{f y} \sin \delta-\mu_{r x}\right)h}mg.\label{eq:frz}
\end{align}

The goal of our controller is driving the vehicle to perform various maneuvers while maintaining it in the drift condition, characterized by a nontrivial sideslip angle $\beta$ (defined in~\eqref{eq:beta})~\cite{zhang2018drift}, e.g., $\beta = -\pi / 3$ during a counter-clockwise move. To illustrate, a vehicle performing a drift maneuver is sketched in Fig.~\ref{fig:drift_circle}.
In particular, we consider the following drift maneuver tasks as they are common training tasks for drift enthusiasts and can therefore serve as benchmarks of our control algorithm:
\begin{itemize}
    \item \textbf{Fixed-circle drifting}: tracking a circle path with fixed center and radius while drifting.
    \item \textbf{Moving-center drifting}: drifting in a circular manner, with a moving desired center.
    \item \textbf{Varying-interaction drifting}: drifting across ground textures with different tire-ground interaction parameters $B,C,D$ (c.f.~\eqref{eq:mf}).
\end{itemize}

\begin{figure}[!htbp]
    \centering
    \includegraphics[width=0.55\columnwidth]{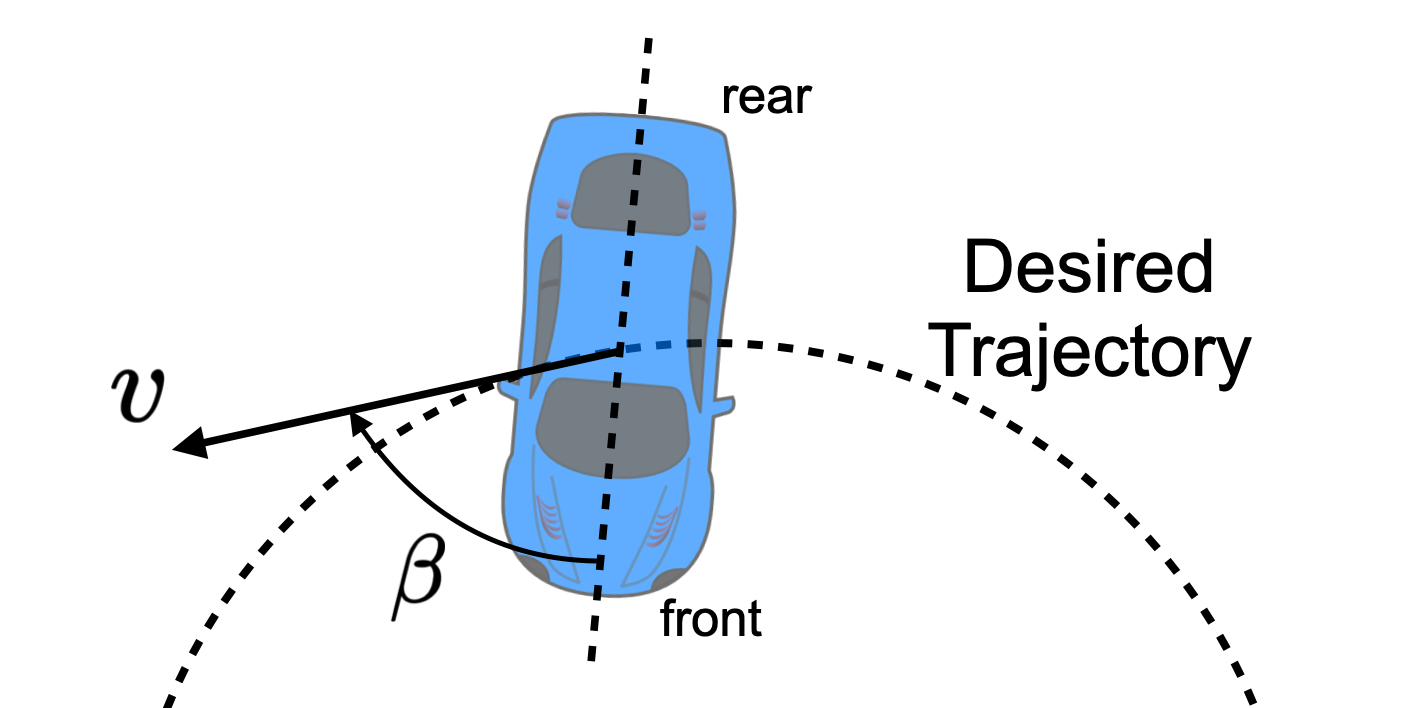}
    \caption{Illustration of vehicle performing a drift maneuver}
    \label{fig:drift_circle}
\end{figure}

\vspace{\spacelength}

\section{Control Architecture}\label{sec:method}

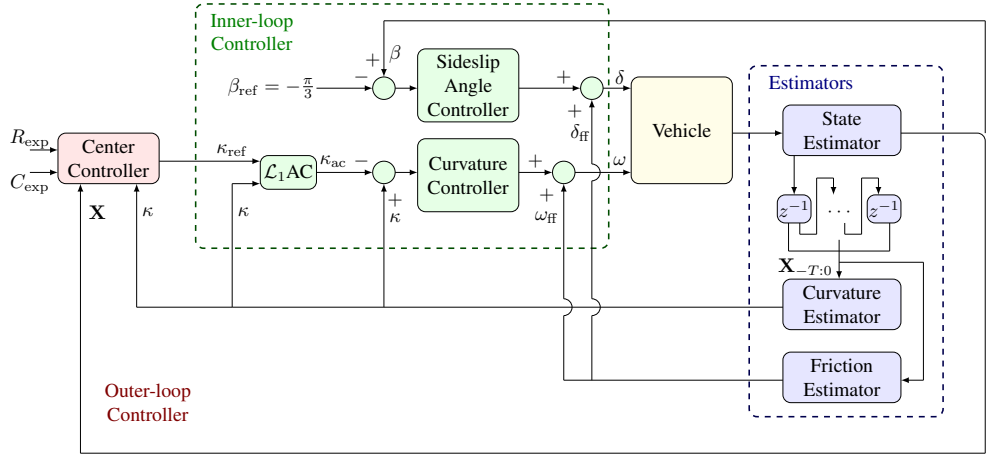
\begin{figure*}[htbp!]
	\centering
	\resizebox{0.75\textwidth}{!}{\usetikzlibrary{arrows}
\tikzset{rounded corners}
\begin{tikzpicture}

\draw [-latex, fill=red!10!white] (-4.8,1.7) rectangle node[text width=2.4cm, align=center]{Center Controller} (-3,0.8) ;
\draw [-latex, fill=green!10!white] (-1.2,1.3) rectangle node{$\mathcal{L}_1$AC} (-0.2,0.7) ;
\draw [-latex](-5.3,1.4) -- (-4.8,1.4);
\draw [-latex](-5.3,1) -- (-4.8,1);

\draw [-latex, fill=blue!10!white] (8.1,2.2) rectangle node[text width=1.6cm, align=center]{State Estimator} (10.2,1.3) ;
\draw [-latex](7.2,1.7) -- (8.1,1.7);

\draw [-latex](-3,1.2) -- (-1.2,1.2);
\draw [-latex, fill=green!10!white] (1,1) ellipse (0.2 and 0.2);
\draw [-latex](-0.2,1) -- (0.8,1);
\draw [-latex](1.2,1) -- (1.6,1);
\draw [-latex, fill=green!10!white] (1.6,1.6) rectangle node[text width=1.6cm, align=center]{Curvature Controller} (3.4,0.3);

\draw [-latex](3.4,1) -- (4,1);
\draw [-latex, fill=green!10!white] (4.2,1) ellipse (0.2 and 0.2);

\draw [-latex](4.4,1) -- (5.4,1);

\draw [-latex,fill=yellow!10!white] (5.4,2.7) rectangle node{Vehicle} (7.2,0.8);
\draw [-latex](3.4,2.5) -- (4.5,2.5);
\draw [-latex, fill=green!10!white] (1.6,1.9) rectangle node[text width=1.6cm, align=center]{Sideslip Angle Controller}(3.4,3.2);
\draw [-latex](1.2,2.5) -- (1.6,2.5);
\draw [-latex, fill=green!10!white] (1,2.5) ellipse (0.2 and 0.2);
\draw [-latex](-0.2,2.5) -- (0.8,2.5);

\draw (11.7,1.7) -- (11.7,3.7);
\draw (11.7,3.7) -- (1,3.7);
\draw [-latex](1,3.7) -- (1,2.7);
\draw [-latex](4.2,-1.25) -- (4.2,0.8);
\draw (4.2,-2.7) -- (8.1,-2.7);

\draw [-latex, fill=blue!10!white] (8.1,-3.1) rectangle node[text width=2.4cm, align=center] {Friction Estimator}(10.2,-2.2);
\draw (10.6,-0.6) -- (9.1,-0.6);

\draw [-latex](-1.7,0.8) -- (-1.2,0.8);
\draw [-latex](1,-1.4) -- (1,0.8);
\draw (-1.7,-1.4) -- (-1.7,0.8);

\draw [-latex, fill=blue!10!white] (8.1,-0.9) rectangle node[text width=1.6cm, align=center]{Curvature Estimator} (10.2,-1.8) ;

\draw (1,-1.4) -- (-1.7,-1.4);
\draw (8.1,-1.4) -- (1,-1.4);
\draw [-latex] (10.6,-2.7) -- (10.2,-2.7);

\draw [-latex, fill=green!10!white] (4.7,2.5) ellipse (0.2 and 0.2);
\draw [-latex](4.9,2.5) -- (5.4,2.5);
\draw [-latex](4.7,1.15) -- (4.7,2.3);
\node at (-1.7,1.4) {$\kappa_{\mathrm{ref}}$};
\node at (-1.5,0.3) {$\kappa$};
\node at (1.2,0.5) {$+$};
\node at (0.6,1.2) {$-$};
\node at (0.1,1.2) {$\kappa_{\mathrm{ac}}$};
\node at (-1,2.5) {$\beta_{\mathrm{ref}}= -\frac{\pi}{3}$};
\node at (0.8,3) {$+$};
\node at (0.6,2.7) {$-$};
\node at (4.2,2.7) {$+$};
\node at (3.9,0.6) {$+$};
\node at (4.4,2.1) {$+$};
\node at (3.7,1.2) {$+$};
\node at (4.5,1.7) {$\delta_{\mathrm{ff}}$};
\node at (5.2,2.7) {$\delta$};
\node at (3.9,0.2) {$\omega_{\mathrm{ff}}$};
\node at (5.2,1.2) {$\omega$};
\node at (1.2,0.2) {$\kappa$};
\node at (1.2,3.1) {$\beta$};
\node at (-5.3,1.6) {$R_{\mathrm{exp}}$};
\node at (-5.3,0.8) {$C_{\mathrm{exp}}$};

\draw (11.7,-4) -- (-4.4,-4);
\draw (11.7,-4) -- (11.7,1.7) -- (10.2,1.7);
\draw (-3.4,-1.4) -- (-1.7,-1.4);
\draw [-latex](-4.4,-4) -- (-4.4,0.8);
\draw [-latex](-3.4,-1.4) -- (-3.4,0.8);
\node at (-3.2,0.3) {$\kappa$};
\node at (-4.1,0.3) {$\mathbf{X}$};
\node at (8.5,-0.7) {$\mathbf{X}_{-T:0}$};

\draw[dashed, green!30!black, thick]  (-2.35,4) rectangle (5,-0.35);
\node[text width=2.4cm, align=center, text=green!50!black] at (-1.35,3.5) {Inner-loop Controller};
\node[text width=2.4cm, align=center, text=red!50!black] at (-3.2,-3.1) {Outer-loop Controller};
\draw[dashed, blue!30!black, thick]  (7.5,2.9) rectangle (10.95,-3.35);

\draw [-latex, fill=blue!10!white] (8,0.6) rectangle node[text width=2.6cm, align=center] {$z^{-1}$}(8.6,0.1);
\draw [-latex, fill=blue!10!white] (9.6,0.6) rectangle node[text width=2.6cm, align=center] {$z^{-1}$}(10.2,0.1);
\node at (9.1,0.3) {\small{$\cdot\cdot\cdot$}};
\draw [-latex] (8.3,1.3) -- (8.3,0.6);
\draw [-latex] (9,0.9) -- (9,0.6);
\draw  (8.4,-0.1) -- (8.4,0.1);
\draw  (8.7,-0.1) -- (8.4,-0.1);
\draw (8.7,0.9)   -- (8.7,-0.1);
\draw (8.7,0.9) -- (9,0.9);

\draw  (9.2,-0.1) -- (9.2,0.1);
\draw  (9.5,-0.1) -- (9.2,-0.1);
\draw (9.5,0.9)   -- (9.5,-0.1);
\draw (9.5,0.9) -- (9.8,0.9);
\draw [-latex] (9.8,0.9) -- (9.8,0.6);

\draw[-latex]  (9.1,-0.2) -- (9.1,-0.9);
\draw  (8.2,-0.4) -- (10,-0.4);
\draw  (10,-0.4) -- (10,0.1);
\draw  (8.2,-0.4) -- (8.2,0.1);

\draw (4.7,1.15) arc (89.0938:-89.0763:0.15);
\draw  (4.7,-1.25) -- (4.7,0.85);

\draw  (10.6,-0.6) -- (10.6,-2.7);
\draw (4.2,-1.25) arc (89.0938:-89.0763:0.15);
\draw (4.7,-1.25) arc (89.7377:-89.7286:0.15);

\draw  (4.7,-1.55) -- (4.7,-2.7);
\draw  (4.2,-1.55) -- (4.2,-2.7);
\node[text=blue!50!black] at (8.6,2.6) {Estimators};
\end{tikzpicture}}
	\caption{Overview of our control architecture}
	\label{fig:big}
\end{figure*}

% - What is the difficulty of drift maneuver control?
%   - nonlinear and difficulty in modelling? (tire: estimation)
%   - time scale? (theta , omega? response)
% - Block Diagram
%   - inner loop (track R) / outer loop (track the center) (why not mpc?)
%     - why estimate curvature 
% - Why adopt this Architecture?
% - how does this Architecture solve the problem?

In this section, we propose our control strategy for drift maneuvers, the architecture of which is presented in Fig.~\ref{fig:big}.
Our controller consists of three main components: i) state, curvature and friction estimators, which provide both low- and high-level descriptions of the current drifting situation; ii) an inner-loop feedforward/feedback controller, which aims to track target curvature $\kappa_{\text{ref}}$ while maintaining the vehicle in sustained drift; iii) an outer-loop controller for center stabilization, which decides a feasible curvature for the inner-loop controller to track, according to the expected center and radius $C_{\exp}, R_{\exp}$ provided by the task specification. A main feature of our control architecture is the estimation of and feedback on the curvature $\kappa$, a key quantity that allows us to decouple of the inner and outer loops, which are responsible for the low-level vehicle stabilization and high-level task completion respectively. In addition, besides the standard state estimator and feedback controller blocks, we adopt a friction estimator for tuning the feedforward signals according to the perceived tire-ground interaction parameters, and an $\mathcal{L}_1$ Adaptive Control~($\mathcal{L}_1$AC) module for optimizing the transient characteristics of the inner loop, both of which serve to improve the adaptive performance of our controller when subject to tire-ground interaction changes. For the rest of this section, we describe each of the three components of our control architecture in detail.

\subsection{Estimators}\label{sec:estimator}

For state estimation, we adopt the Kalman Filter~(KF) to fuse position and angular velocity measurements from both on-board sensors, e.g., Inertial Measurement Unit~(IMU), and off-board sensors, e.g., the motion capture system. Taking communication delay into account, we implement the KF in an asynchronous manner~\cite{async_kf}, performing an update upon the arrival of each new observation from any sensor.

For both curvature and friction estimation, we utilize state and input data (defined in~\eqref{eq:xu}) from a recent time window of $T$ steps, denoted by $\{\mathbf{X}^{(i)}, \mathbf{U}^{(i)}\}_{i=-T+1}^0$.

For curvature estimation, we fit a circle centered at $\left(x_0,y_0\right)$ with radius $R$ that approximately crosses those $T$ points, by solving the following nonlinear optimization problem:
\begin{align*}
    \arg \min\limits_{x_0,y_0,R} & \sum\limits_{i=-T+1}^{0}\left({R_\text{geo}^{(i)}}-{R}\right)^2+\left({R_\text{kin}^{(i)}}-R\right)^2, \\
    \text{s.t. }& R_\text{geo}^{(i)} = \sqrt{\left(x^{(i)}-x_0\right)^2+\left(y^{(i)}-y_0\right)^2}, \\
    & R_\text{kin}^{(i)} = \frac{\sqrt{\left(\dot{x}^{(i)}\right)^2+\left(\dot{y}^{(i)}\right)^2}}{\left|\Dot{\psi}^{(i)}\right|},
\end{align*}
and obtain the curvature $\kappa$ according to $\kappa = R^{-1}$. Notice that due to the minimization of the squared error, the optimal radius follows:
\begin{equation}
     R = \sum\limits_{i=-T+1}^{0}\frac{R_\text{geo}^{(i)}+R_\text{kin}^{(i)}}{2T}.
\end{equation}
In other words, our algorithm fuses geometric information with kinematics information (denoted by $R_\text{geo}$ and $R_\text{kin}$ respectively in the above formulation), which leverages the potential of both position and angular velocity sensors. On the other hand, most existing circle fitting algorithms use purely geometric information~\cite{Ali2009, Kasa1976ACF}. 

For friction estimation, we observe that it is very challenging to identify the parameters $B,C,D$ in a timely fashion. As a compromise, we simplify the friction estimation by assuming a constant friction coefficient $\mu = D \sin \left(C \operatorname{atan} \left(B s_{f}\right)\right)= D \sin \left(C \operatorname{atan} \left(B s_{r}\right)\right)$ between the wheels during a short time window (c.f.~\eqref{eq:mf}), which agrees with our empirical observations. Based on this assumption, the friction estimator attempts to find $\mu$ that best explains the observed data by solving the following optimization problem:
\begin{align*}
    \underset{\mu}{\operatorname{min}}\; \sum_{i=-T+1}^{-1} \ell\left(f\left(\mathbf{X}^{(i)}, \mathbf{U}^{(i)}; {\mu}\right), \mathbf{X}^{(i+1)}\right),
    % \label{eq:loss}
\end{align*}
where $\ell$ denotes the $\ell_2$ loss, $f$ denotes the discretized dynamics equations~\eqref{eq:xddot}-\eqref{eq:mf} with $D \sin \left(C \arctan \left(B s_{i}\right)\right)$ replace by $\mu$. Once the estimated friction coefficient $\mu$ is obtained, the corresponding optimal feedforward control inputs $\delta_{\text{ff}}, \omega_{\text{ff}}$ can be found by looking up a table of drift equilibria collected offline.

\subsection{Inner-loop Controller}

The basis of our inner-loop controller is two feedback control loops: the first control loop stabilizes the sideslip angle $\beta$ by tuning the front wheel steering angle $\delta$. The goal of this controller is to keep the vehicle in a sustained drift status, by maintaining the sideslip angle $\beta$ close to a constant, e.g., $\beta_{\mathrm{ref}} = -\pi/3$. The other feedback loop controls the curvature $\kappa$ by tuning the wheel rotational speed $\omega$ in order to track the reference signal $\kappa_{\mathrm{ref}}$ from the outer control loop. In both simulation and experiment, PID controllers are sufficient for both feedback control loops to achieve satisfactory control performance. Feedforward signals $\delta_{\text{ff}}, \omega_{\text{ff}}$, which depend on tire-ground friction coefficient $\mu$ (see subsection~\ref{sec:estimator}) and target curvature, are supplemented to the feedback controller to speed up the transient process.

To further compensate for potential changes in the tire-ground interaction, in addition to the feedforward scheme, we adopt an $\mathcal{L}_1$ adaptive control~\cite{4282999,hindman2007designing,michini2009l1} component, the structure of which is shown in the Fig.~\ref{fig:l1ac}, where $M(s)$ represents the reference model with ideal curvature response properties, $C(s)$ is a Low-Pass Filter~(LPF), which limits the bandwidth of reference curvature signal $\kappa_{\mathrm{ac}}$ provided to the feedforward/feedback controller to avoid high frequency oscillation while maintaining a fast adaptive rate, and $\Gamma$ is the adaptive gain.
%The most important advantage of $\mathcal{L}_1$ adaptive control is the decoupling of control loop and the estimation loop, so that both transient performance and robustness can be guaranteed. 			% Need to clarify

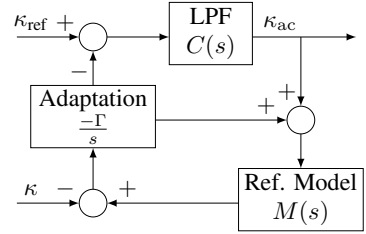
\begin{figure}[htbp!]
	\centering
	\resizebox{0.6\columnwidth}{!}{\usetikzlibrary{arrows}
\begin{tikzpicture}

\draw [-latex] (-1.1,13.8) rectangle node[text width=1cm, align=center]{LPF $C(s)$} (0.1,12.9) ;
\draw [-latex] (-2.2,13.3) ellipse (0.2 and 0.2);
\draw [-latex](-2.2,11.1) -- (-2.2,11.7);
\draw [-latex](0.1,13.3) -- (1.6,13.3);
\draw [-latex] (-0.1,11.4) rectangle node[text width=2cm, align=center]{Ref. Model $M(s)$} (1.7,10.5);
\draw [-latex] (-2.2,10.9) ellipse (0.2 and 0.2);
\draw [-latex](0.8,13.3) -- (0.8,12.3);

\draw [-latex](-2,13.3) -- (-1.1,13.3);
\draw [-latex] (0.8,12.1) ellipse (0.2 and 0.2);
\draw [-latex](-3.3,13.3) -- (-2.4,13.3);

\draw [-latex](-2.2,12.6) -- (-2.2,13.1);
\draw [-latex](-0.1,10.9) -- (-2,10.9);

\draw [-latex] (-3.1,11.7) rectangle node[text width=1.6cm, align=center] {Adaptation $\frac{-\Gamma}{s}$}(-1.3,12.6);

\draw [-latex](-3.3,10.9) -- (-2.4,10.9);
\draw [-latex](0.8,11.9) -- (0.8,11.4);

\draw [-latex](-1.3,12.1) -- (0.6,12.1);

\node at (-3.1,13.5) {$\kappa_{\text{ref}}$};

\node at (-2.6,11.1) {$-$};
\node at (-3.1,11.1) {$\kappa$};

\node at (-2.6,13.5) {$+$};

\node at (0.3,12.3) {$+$};

\node at (0.6,12.5) {$+$};
\node at (-1.7,11.1) {$+$};
\node at (-2.4,12.8) {$-$};

\node at (0.5,13.5) {$\kappa_{\mathrm{ac}}$};
\end{tikzpicture}}
	\caption{Structure of our applied $\mathcal{L}_1$ adaptive controller}
	\label{fig:l1ac}
\end{figure}

% Hindman et al.~\cite{hindman2007designing} analyze $\mathcal{L}_1$ adaptive output feedback control in detail and then give the specific expression of the closed-loop response:$y(s)=H(s)C(s)r(s)+H(s)(1-C(s))d(s)$, where $H(s)=\frac{A(s)M(s)}{C(s)A(s)+(1-C(s))M(s)}$.

By designing the reference model $M(s)$, the LPF $C(s)$ and the adaptive gain $\Gamma$, the ideal closed-loop response $H(s)C(s)$ with respect to the curvature $\kappa$ can be obtained.

Readers are referred to Michini et al.~\cite{michini2009l1} for the full technical details of the $\mathcal{L}_1$ adaptive controller design.

\subsection{Outer-loop Controller}

The goal of the outer control loop is to ensure circumnavigation around either a fixed or moving target, by controlling the curvature of vehicle. Circumnavigation using only bearing or distance information has been studied by \cite{6705614,zheng2015enclosing} and \cite{DONG2020108932} respectively. In our setup, we consider only bearing information, which is sufficient for the circumnavigation of a drifting vehicle. We assume w.l.o.g. that the desired path is counterclockwise. As such, we adopt the following control law:
\begin{align*}
	\kappa_{\mathrm{ref}}=  \kappa_{0}(1+\gamma\cos(\phi)),
\end{align*}
where $\gamma>0$ is the curvature adjustment rate, $\kappa_0=1/R_{\exp}$ is the desired curvature and $\phi$ is defined as:
\begin{equation*}
\phi \triangleq \operatorname{atan2}(\dot{y},\dot{x}) - \operatorname{atan2}(y-C_y, x-C_x), 
\end{equation*}
where $(C_x,C_y)$ is the center of the desired circle. By a change of coordinate, as is shown in Fig.~\ref{fig:move_center}, the kinematics of our vehicle can be characterized as 
\begin{align*}
	\dot{d} &= v \cos(\phi),\;\dot{\phi} = v\kappa- \frac{v}{d}\sin(\phi),
\end{align*}
where $d$ is the distance between the vehicle to the desired center, and $v = \sqrt{\dot{x}^2+\dot{y}^2}$ is the velocity.

\begin{figure}[htbp!]
	\centering
	\resizebox{0.5\columnwidth}{!}{\begin{tikzpicture}

\draw  (-0.5,-1.5) ellipse (2 and 2);
\draw[-latex] (-0.5,-1.51) -- (-1.7,0.1);

\node at (-0.9,-0.4) {$R_{\exp}$};
\node at (-0.43,-1.82) {$(C_x, C_y)$};
E

\node at (2.7,-0.5) {$(x,y)$};
\fill (2.3,-0.2) circle (1pt);
\fill (-0.5,-1.5) circle (1pt);
\draw (-0.5,-1.5) -- (2.3,-0.2) node {};
\node at (0.91,-0.66) {$d$};
\draw[-latex] (2.3,-0.2) -- (1.75,0.35);
\node at (1.64,0.13) {$v$};

\node at (2.5,0.5) {$\phi$};

\draw[-latex] (2.7,-0.01) arc (6.8802:152.0211:0.3899);

\draw[dashed] (2.3,-0.2) -- (3.2,0.217);
\end{tikzpicture}}
	\caption{Illustration of circumnavigation around an expected center}
	\label{fig:move_center}
\end{figure}
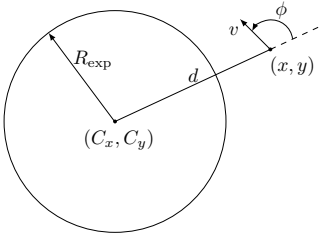

Assume that the inner control loop is perfect, then $\kappa = \kappa_{\mathrm{ref}}$. In that case, we arrive at the following differential equations:
\begin{align}
	\dot{d} &= v \cos(\phi), \label{eq:dubins} \\
	\dot{\phi} &= v\left(\kappa_0(1+\gamma\cos(\phi))- \frac{1}{d}\sin(\phi)\right).\label{eq:dubins2}
\end{align}

Further assuming that the velocity $v$ is bounded away from $0$, then it is easy to see that the only equilibrium of the above differential equation is 
\begin{align*}
	(\phi, d) = (\pi/2+2k\pi,1/\kappa_0),
\end{align*}
which correspond to the desired circular trajectory. We can further prove that the actual trajectory must converge to the desired trajectory under appropriate $\gamma$, as is formally stated in the following theorem.
\begin{theorem}
Assume that $\kappa_0 > 0, \gamma >0$ and that $v>0$ at any time, then the following equilibrium is globally asymptotically stable:
	\begin{equation*}
	\phi = \pi/2+2k\pi,\, d = 1/\kappa_0.
	\end{equation*}
\end{theorem}

\begin{proof}
	Consider the following Lyapunov function candidate,
	\begin{align*}
		V = \frac12(d - \kappa_0^{-1}) ^ 2 + \kappa_0^{-1}  d  (1 - \sin\phi),
	\end{align*}
	whose derivative is
	\begin{align*}
		\dot{V} = (d - \kappa_0^{-1} \sin\phi) \dot{d} - \kappa_0^{-1} d \cos\phi \,\dot{\phi}.
	\end{align*}
	Substituting in~\eqref{eq:dubins}, \eqref{eq:dubins2}, we have
	\begin{align*}
		\dot{V} &= v \left[ (d - \kappa_0^{-1} \sin\phi)\cos\phi - d\cos\phi(1+\gamma\cos\phi) + \right. \\ &\left.\kappa_0^{-1} \cos\phi \sin\phi \right] \\
		& = -v\gamma \cos^2\phi \leq 0.
	\end{align*}
	It can be seen that $\dot{V} = 0$ only when $\cos\phi = 0$.

	We next show that the largest invariance set in $\{(\phi, d) | \dot{V} = 0\}$ is $\{(\pi/2+2k\pi, 1/\kappa_0)\}$. Indeed, from $\cos\phi = 0$ we have $\sin\phi \in {-1, +1}$. Under the above condition, from~\eqref{eq:dubins2} we see $\dot{\phi} = 0$ if and only if $\sin \phi = 1, 1/d = \kappa_0$, i.e., $(\phi, d) = (\pi/2+2k\pi, 1/\kappa_0)$. According to LaSalle's invariance theorem~\cite[Theorem 4.4]{khalil}, the trajectory of $(\phi, d)$ converges to $(\pi/2+2k\pi, 1/\kappa_0)$, i.e., the equilibrium is globally asymptotically stable.
\end{proof}

%\subsection{State Estimation}

%\input{state-estimation.tex}

% \subsection{Curvature Estimation}

% \input{curvature-estimation.tex}

% \subsection{Center Control Design}

% \input{center-control.tex}

% \subsection{Feedback Control Design}

% \input{feedback.tex}

% \subsection{Online Data-driven Feedforward Design}

% \input{feedforward.tex}

% \subsection{$\mathcal{L}_1$ Adaptive Control}

% \input{L1AC.tex}

\vspace{\spacelength}

\section{Simulation and Experiment}\label{sec:result}

\subsection{Simulation Results}

We benchmark the performance of our proposed drift control strategy using the three drift maneuver tasks defined in Section~\ref{sec:model}. In this subsection, we report the simulation results based on a simulator of the vehicle dynamics model~\eqref{eq:xddot}-\eqref{eq:frz}.

The first task we consider is fixed-circle drifting, where the expected center is $(0,0)$ and the expected radius is $10\mathrm{m}$. The controller is expected to initialize the drift starting from zero velocity, and track the target circle while maintaining the drift afterwards. The simulation results for this task are presented in Fig.~\ref{fig:fix}. We can observe from the figure that our proposed controller is effective in both drift initialization and target tracking. In particular, unlike previous methods on drift maneuver control~\cite{zhang2018drift,Jelavic2017AutonomousDP}, which switch between an expert-aided open-loop controller for initializing the drift and a feedback controller for maintaining the drift, we use the same controller for initializing and maintaining the drift, without resort to any open-loop design.
We can observe from Fig.~\ref{fig:fix-p1} that using our unified controller, the sideslip angle stabilizes around the reference value in about (10 seconds), which corresponds to traveling through an arc of less than $360^\circ$ (the arc in green in Fig.~\ref{fig:fix-p3}). This demonstrates the effectiveness of our controller in a large portion of the state space. Furthermore, we can observe from Fig.~\ref{fig:fix-p2} that after entering a sustained drift, the curvature tracking error converges to zero, which illustrates the effectiveness of hierarchical control architecture based on curvature inference and feedback. Finally, we can observe from Fig.~\ref{fig:fix-p4} that over the entire drift initialization and maintenance process, the maximum tracking error relative to the expected radius is less than $15\%$, which proves that our proposed drift control scheme is capable of tracking a drift trajectory with adequate precision.

\begin{figure}
    \centering
    \begin{subfigure}{0.23\textwidth}
        \centering
        \includegraphics[width=\textwidth]{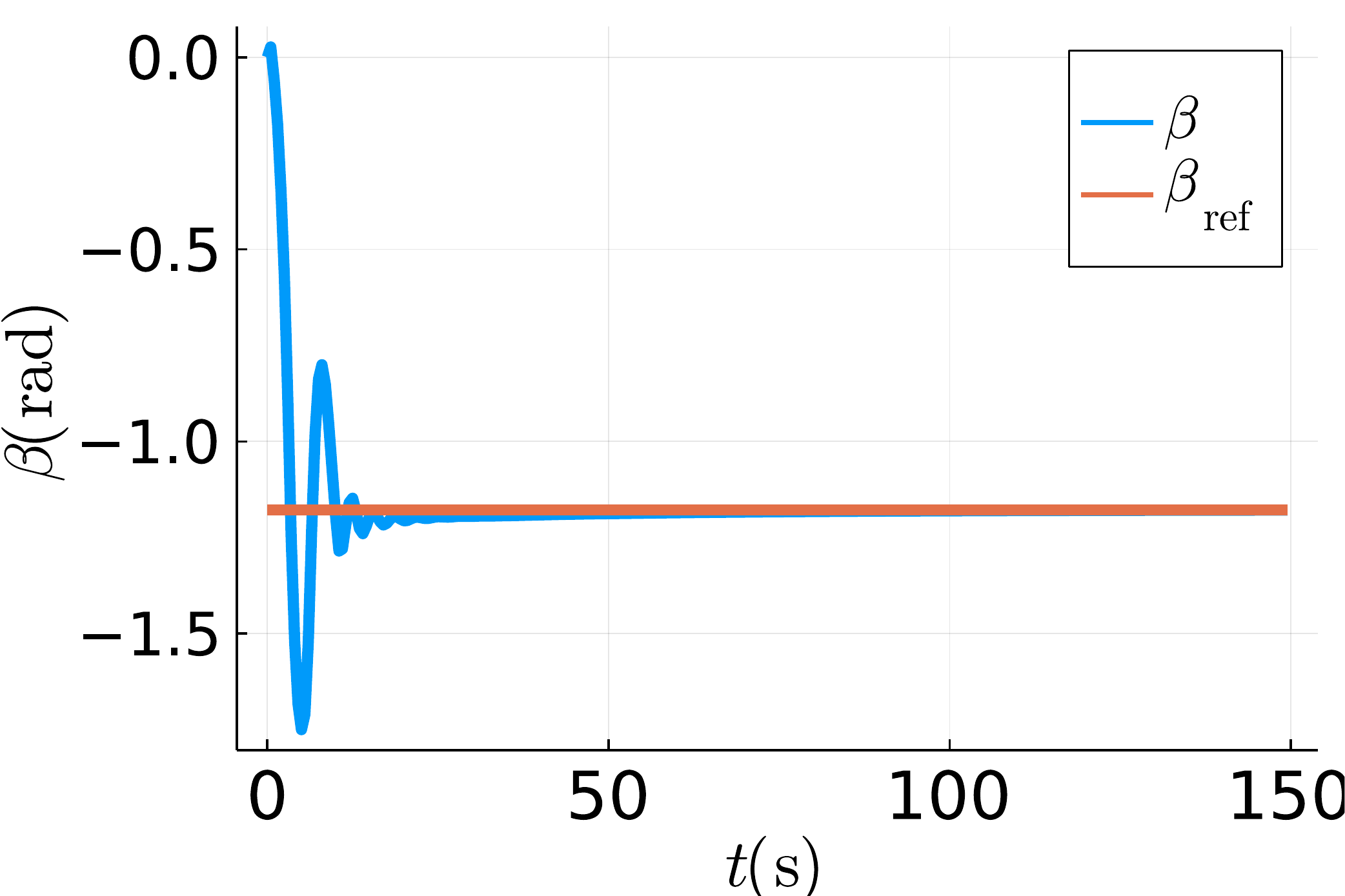}
        \caption{Tracking performance of sideslip angle $\beta$}\label{fig:fix-p1}
    \end{subfigure}
    \begin{subfigure}{0.23\textwidth}
        \centering
        \includegraphics[width=\textwidth]{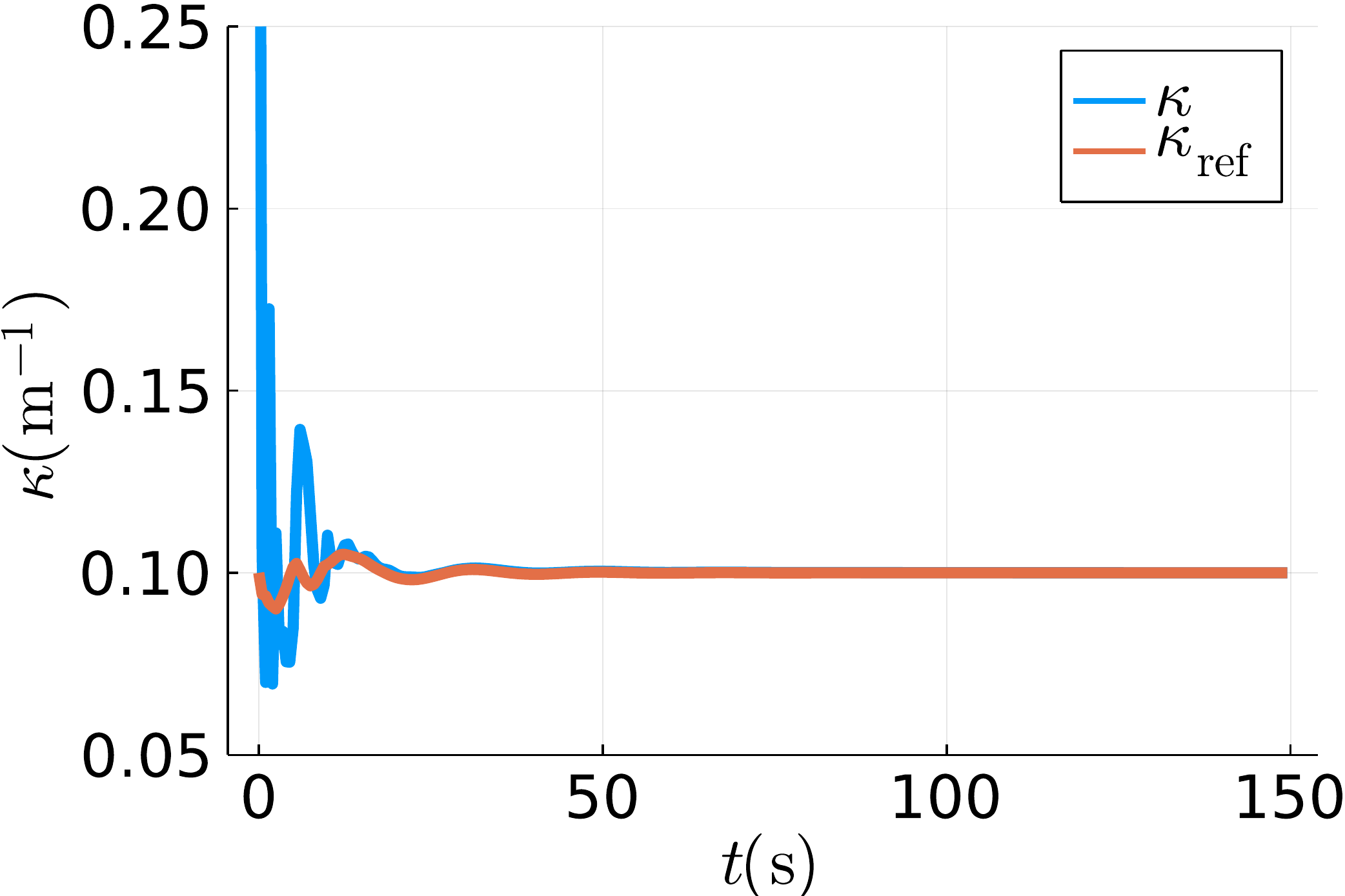}
        \caption{Tracking performance of curvature $\kappa$}\label{fig:fix-p2}
    \end{subfigure} \\
    \begin{subfigure}{0.23\textwidth}
        \centering
        \includegraphics[width=\textwidth]{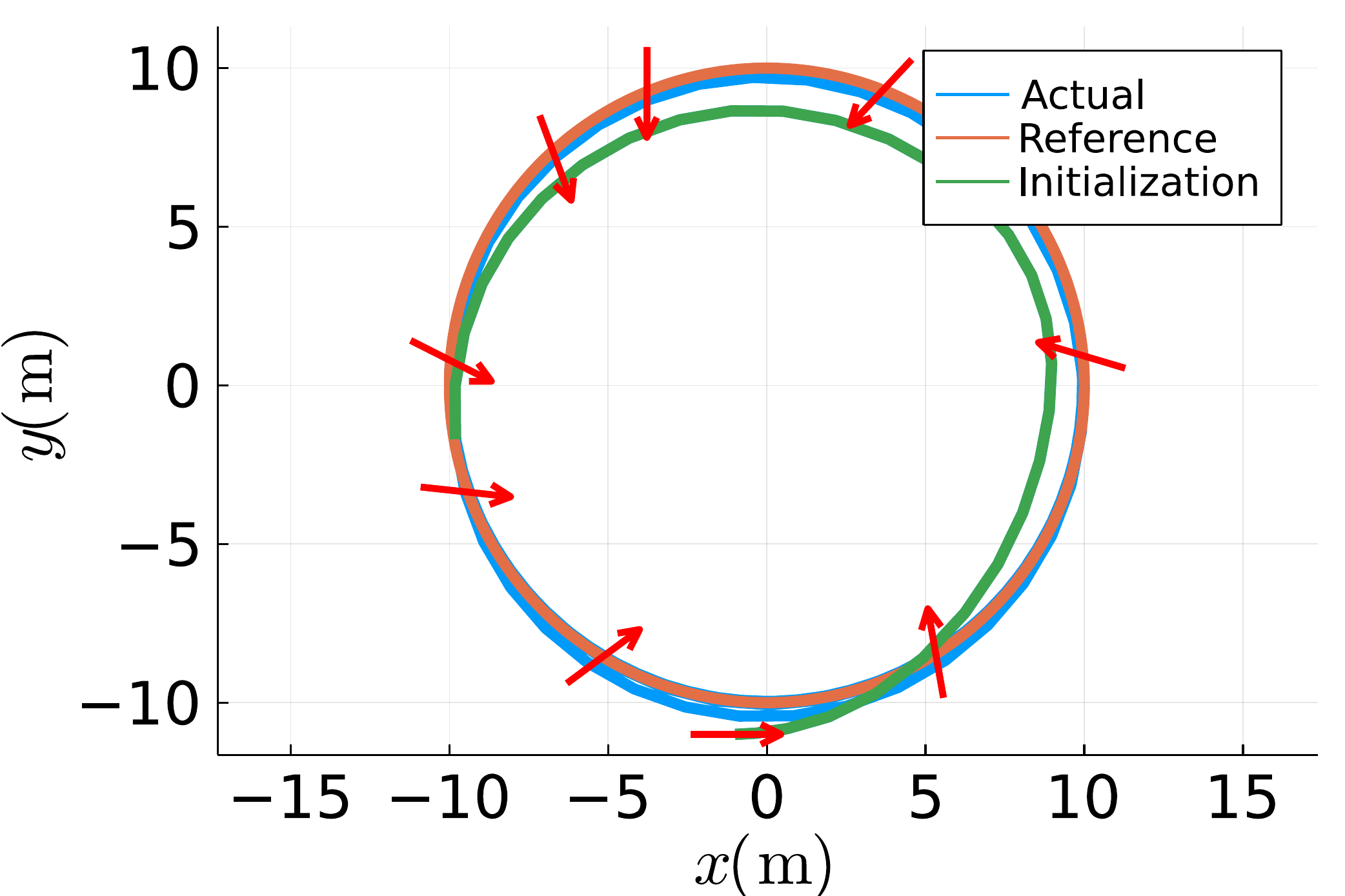}
        \caption{Reference and actual trajectories (red arrows indicate the orientation of the vehicle)}\label{fig:fix-p3}
    \end{subfigure}
    \begin{subfigure}{0.23\textwidth}
        \centering
        \includegraphics[width=\textwidth]{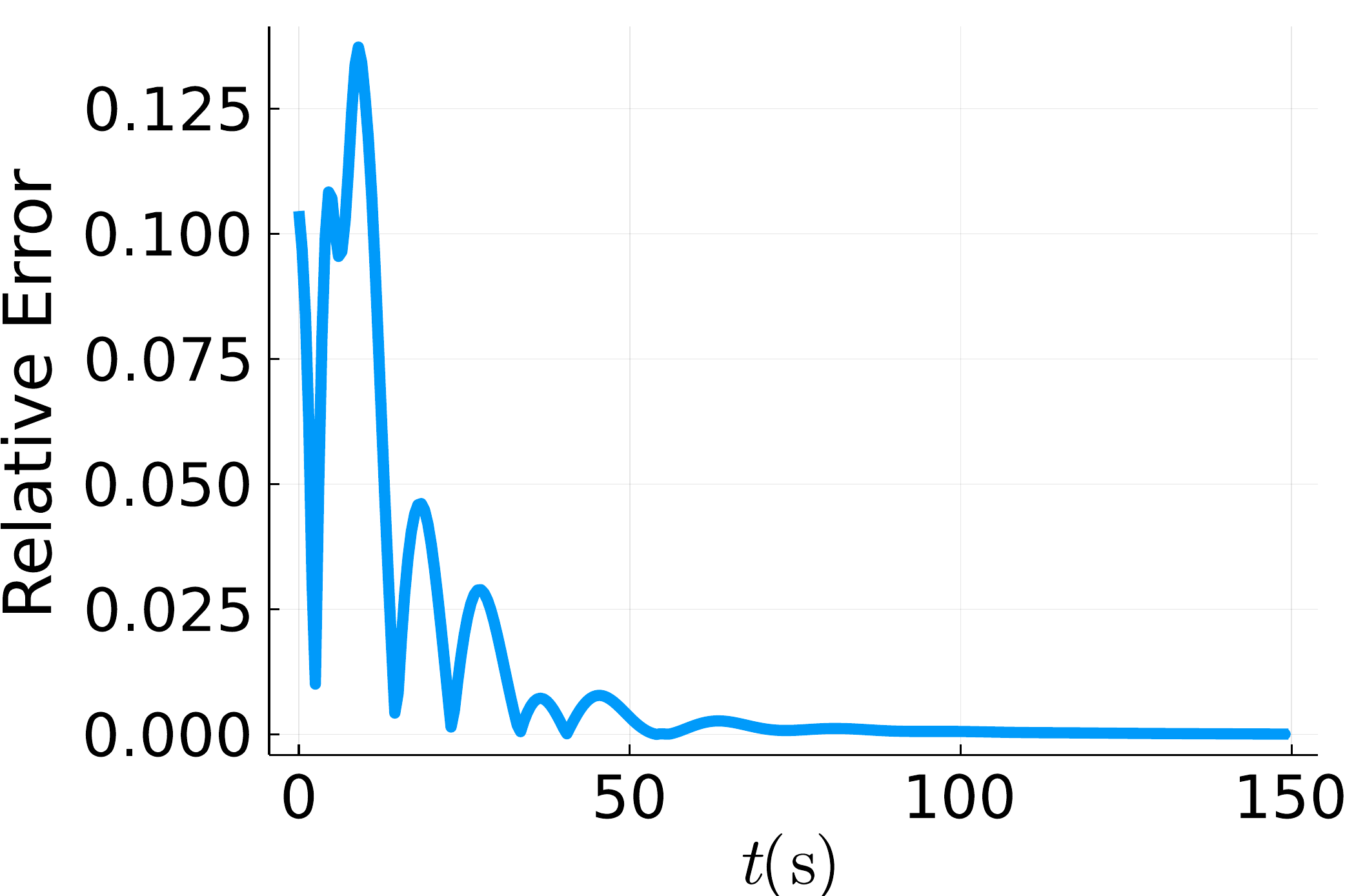}
        \caption{Relative tracking error}\label{fig:fix-p4}
    \end{subfigure}
    \caption{Simulation results for fixed-circle drifting}\label{fig:fix}
\end{figure}

The second task we consider is moving-center drifting, where the expected drifting center itself moves along a larger circle. This can be intuitively visualized as a satellite (the vehicle) orbiting a planet (the expected drifting center), which in turn orbits a star (the center of the larger circle). In this task, the continuous adjustment of drift trajectory curvature is required for the vehicle to follow the moving center smoothly.
We choose the orbit of the expected drifting center to be a circle centered at $(0, 0)$ with radius $15\mathrm{m}$, the movement speed of the drifting center to be $0.131 \mathrm{m/s}$, and the expected drifting radius to be $10 \mathrm{m}$.
The simulation results for this task are presented in Fig.~\ref{fig:move}. It can be visually checked from Fig.~\ref{fig:move-p3} that our proposed controller is able to drive the vehicle to move in the desired manner. We can observe from Fig.~\ref{fig:move-p2} that the tracking of a moving center realized by the cooperation between our outer-loop controller, which sets the reference curvature $\kappa_{\mathrm{ref}}$ to vary periodically, and our inner-loop controller, which makes the actual curvature $\kappa$ track a non-stationary reference agilely. 

\begin{figure}
    \centering
    \begin{subfigure}{0.23\textwidth}
        \centering
        \includegraphics[width=\textwidth]{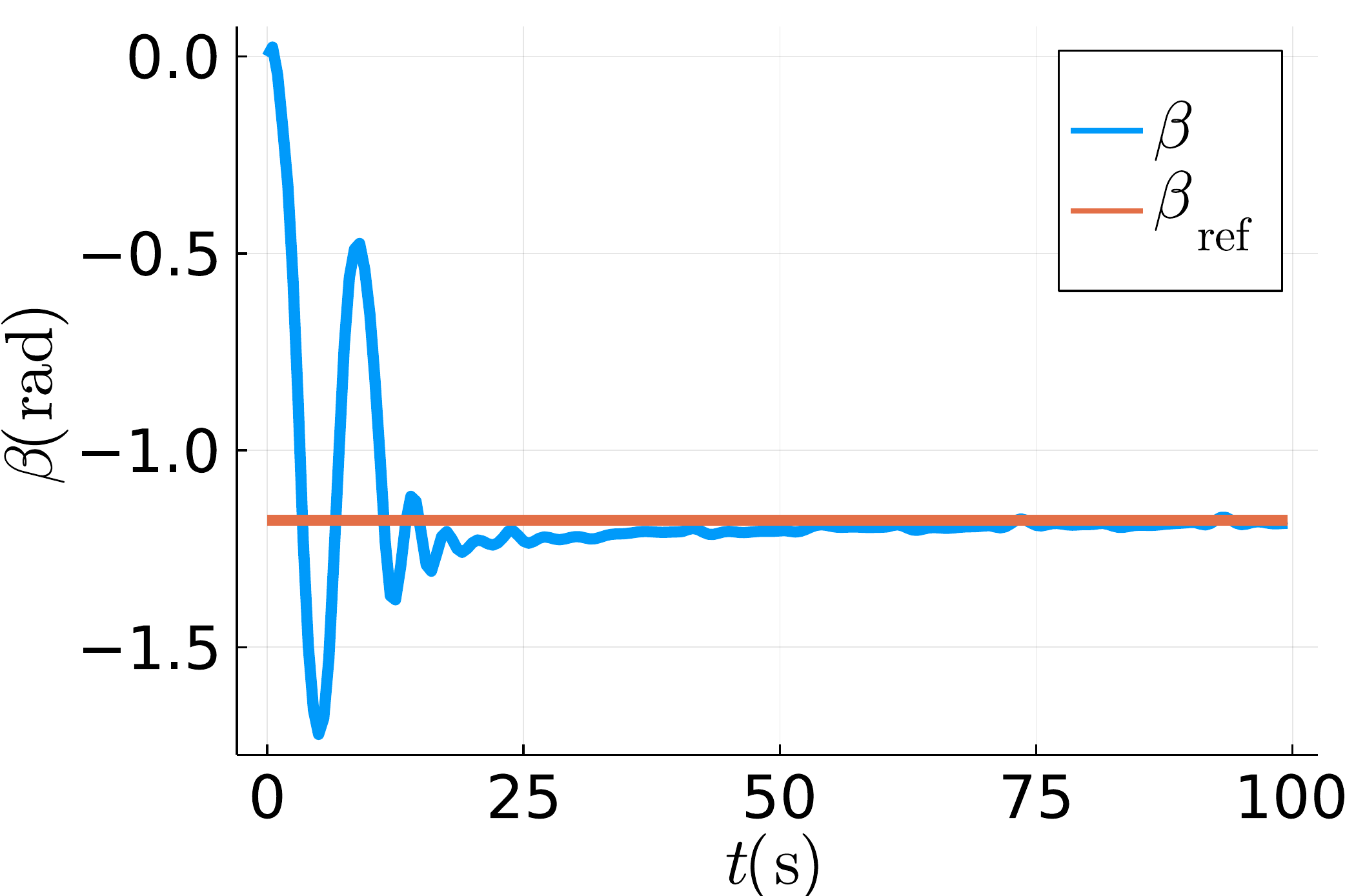}
        \caption{Tracking performance of sideslip angle $\beta$}\label{fig:move-p1}
    \end{subfigure}
    \begin{subfigure}{0.23\textwidth}
        \centering
        \includegraphics[width=\textwidth]{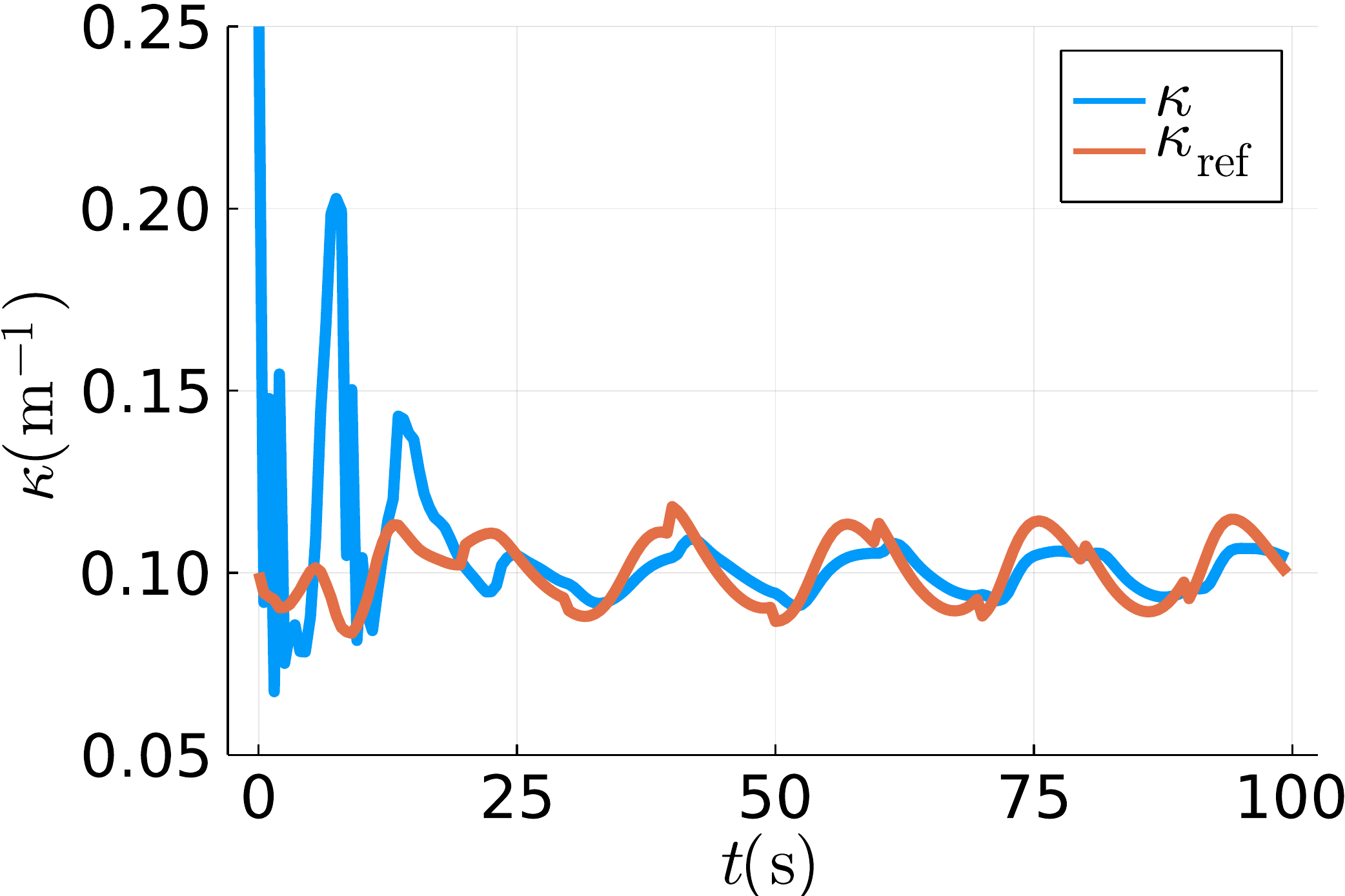}
        \caption{Tracking performance of curvature $\kappa$}\label{fig:move-p2}
    \end{subfigure} \\
    \begin{subfigure}{0.23\textwidth}
        \centering
        \includegraphics[width=\textwidth]{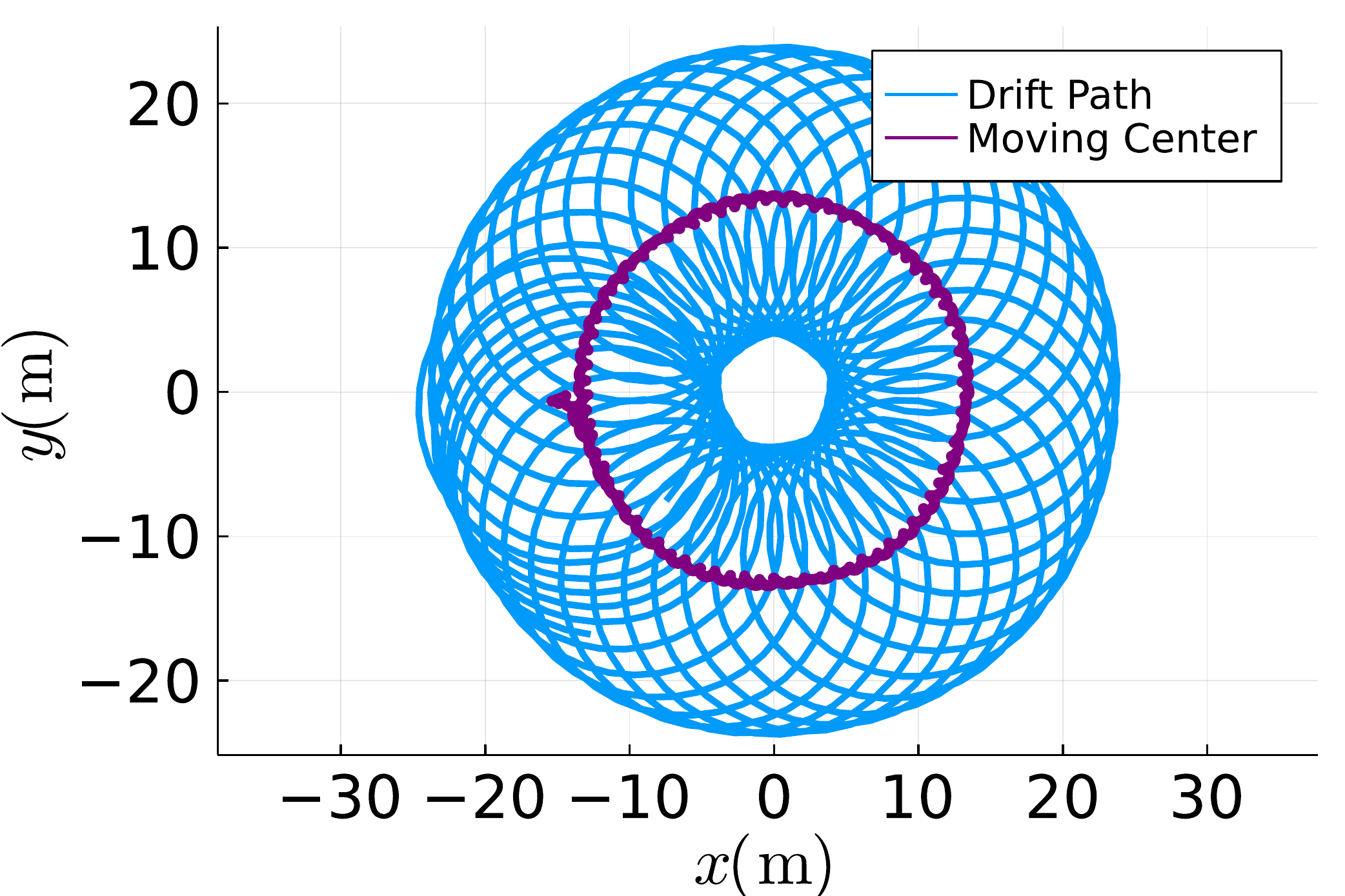}
        \caption{Moving center and drift trajectories}\label{fig:move-p3}
    \end{subfigure}
    \begin{subfigure}{0.23\textwidth}
        \centering
        \includegraphics[width=\textwidth]{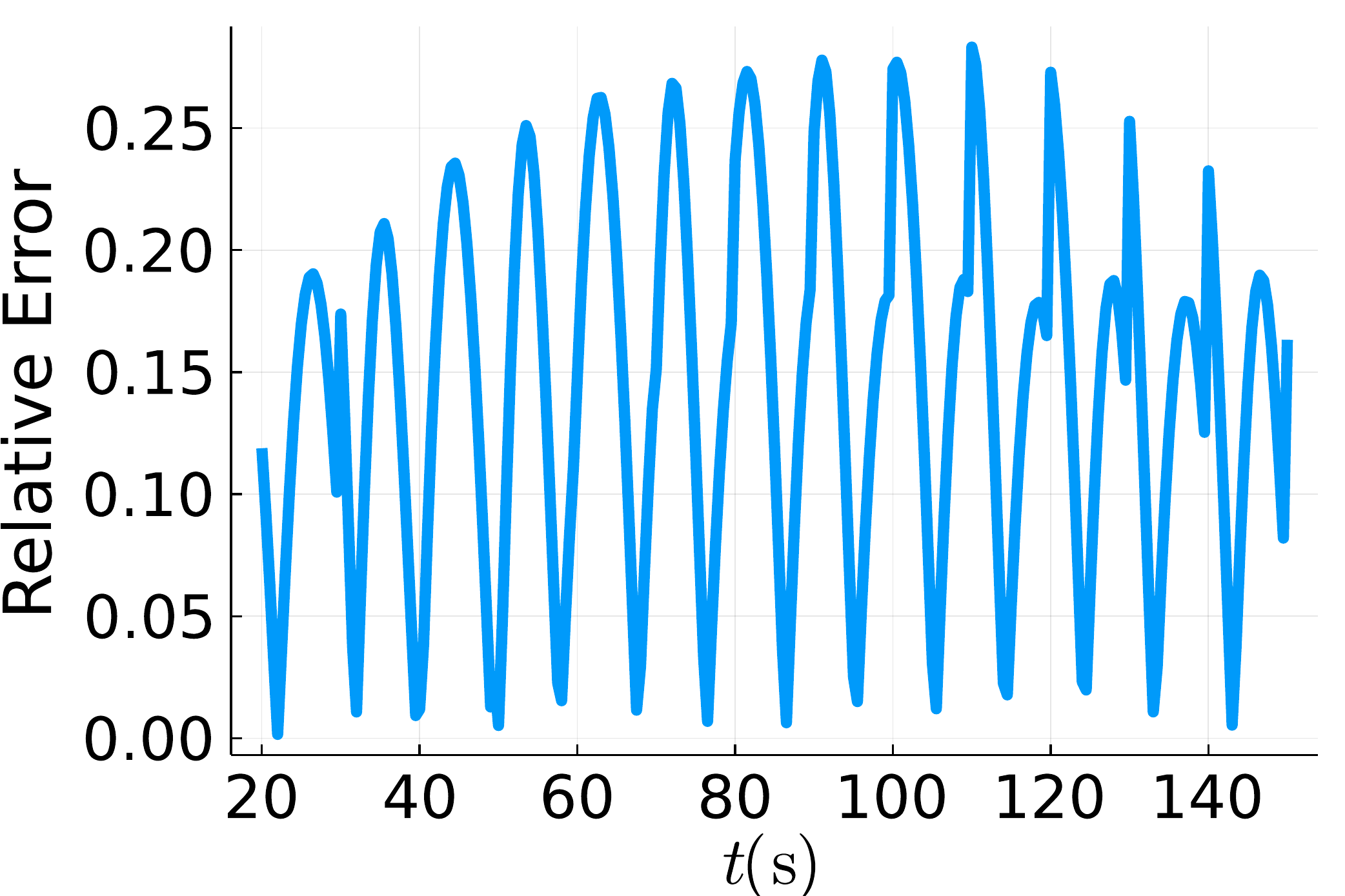}
        \caption{Relative tracking error}\label{fig:move-p4}
    \end{subfigure}
    \caption{Simulation results for moving-center drifting}\label{fig:move}
\end{figure}

The third task we consider is varying-interaction drifting, where the tire-ground interaction parameters $B,C,D$ (c.f.~\eqref{eq:mf}) undergo a sudden change during drifting. In particular, those parameters are changed from $B=5,C=2,D=0.3$ to $B=4, C=2, D=0.15$ at time $t = 200 \mathrm{s}$, which correspond to the friction coefficient $\mu$ changing from around $0.12$ to around $0.07$, which implies a loss of traction of $40\%$, posing challenge to the drift controller. The simulation results for this task are presented in Fig.~\ref{fig:vary}. We can observe from the figure that the curvature $\kappa$ and the sideslip angle $\beta$ undergo a sudden change due to the loss of friction, but our inner-loop controller quickly brings them back to the reference values. The maximum tracking error throughout the process is about $30\%$.

\begin{figure}
    \centering
    \begin{subfigure}{0.23\textwidth}
        \centering
        \includegraphics[width=\textwidth]{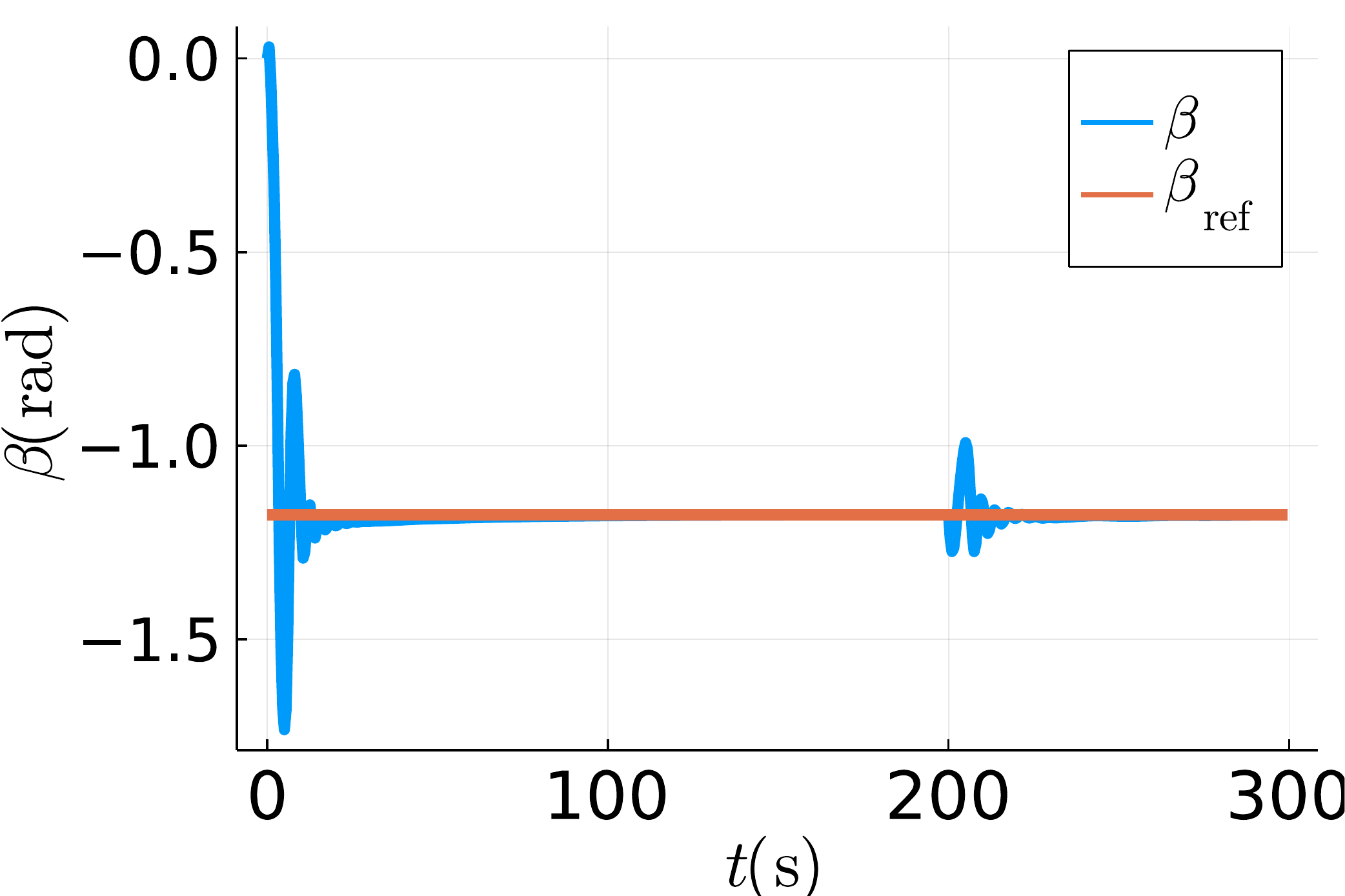}
        \caption{Tracking performance of sideslip angle $\beta$}\label{fig:vary-p1}
    \end{subfigure}
    \begin{subfigure}{0.23\textwidth}
        \centering
        \includegraphics[width=\textwidth]{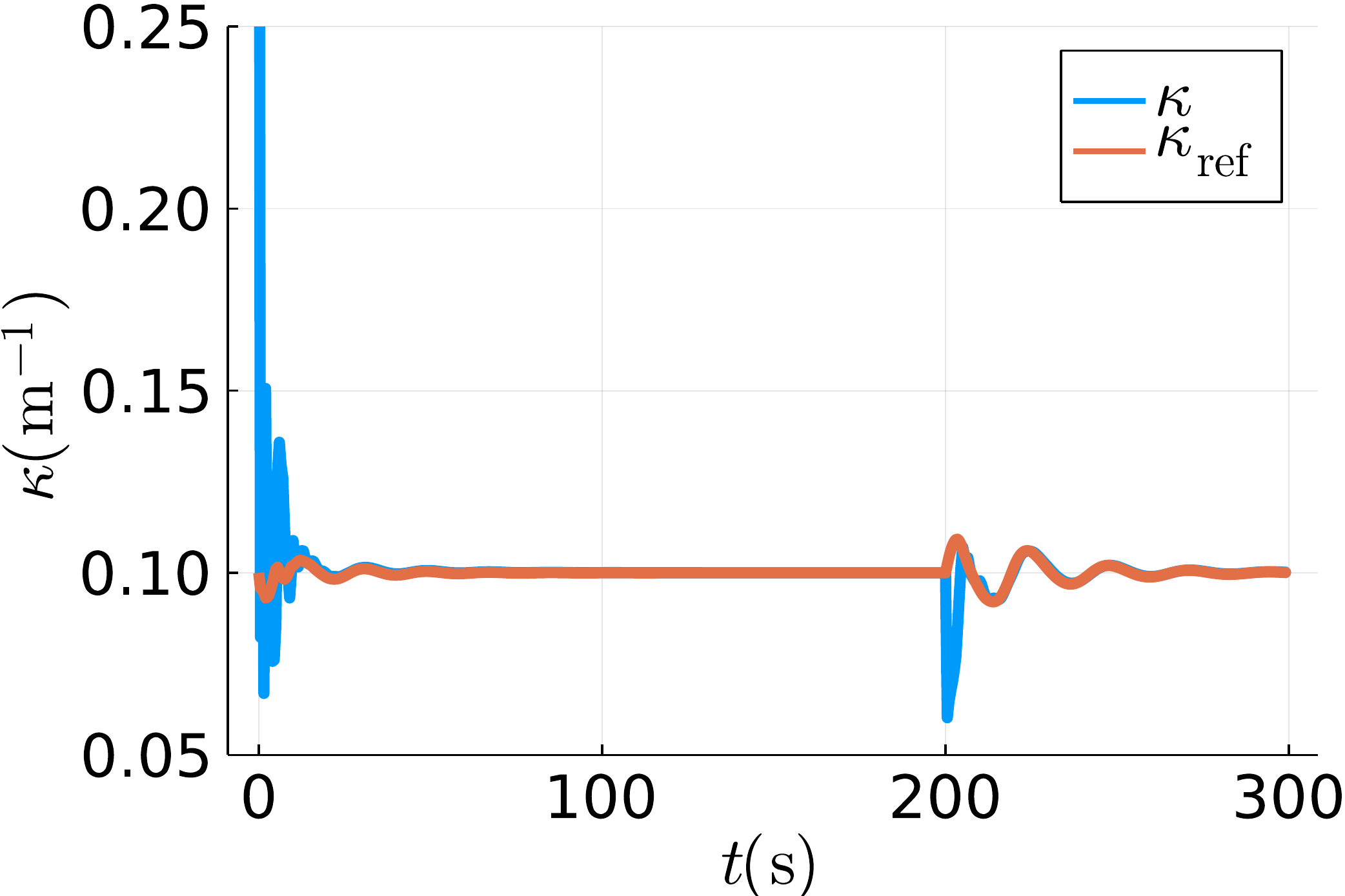}
        \caption{Tracking performance of curvature $\kappa$}\label{fig:vary-p2}
    \end{subfigure} \\
    \begin{subfigure}{0.23\textwidth}
        \centering
        \includegraphics[width=\textwidth]{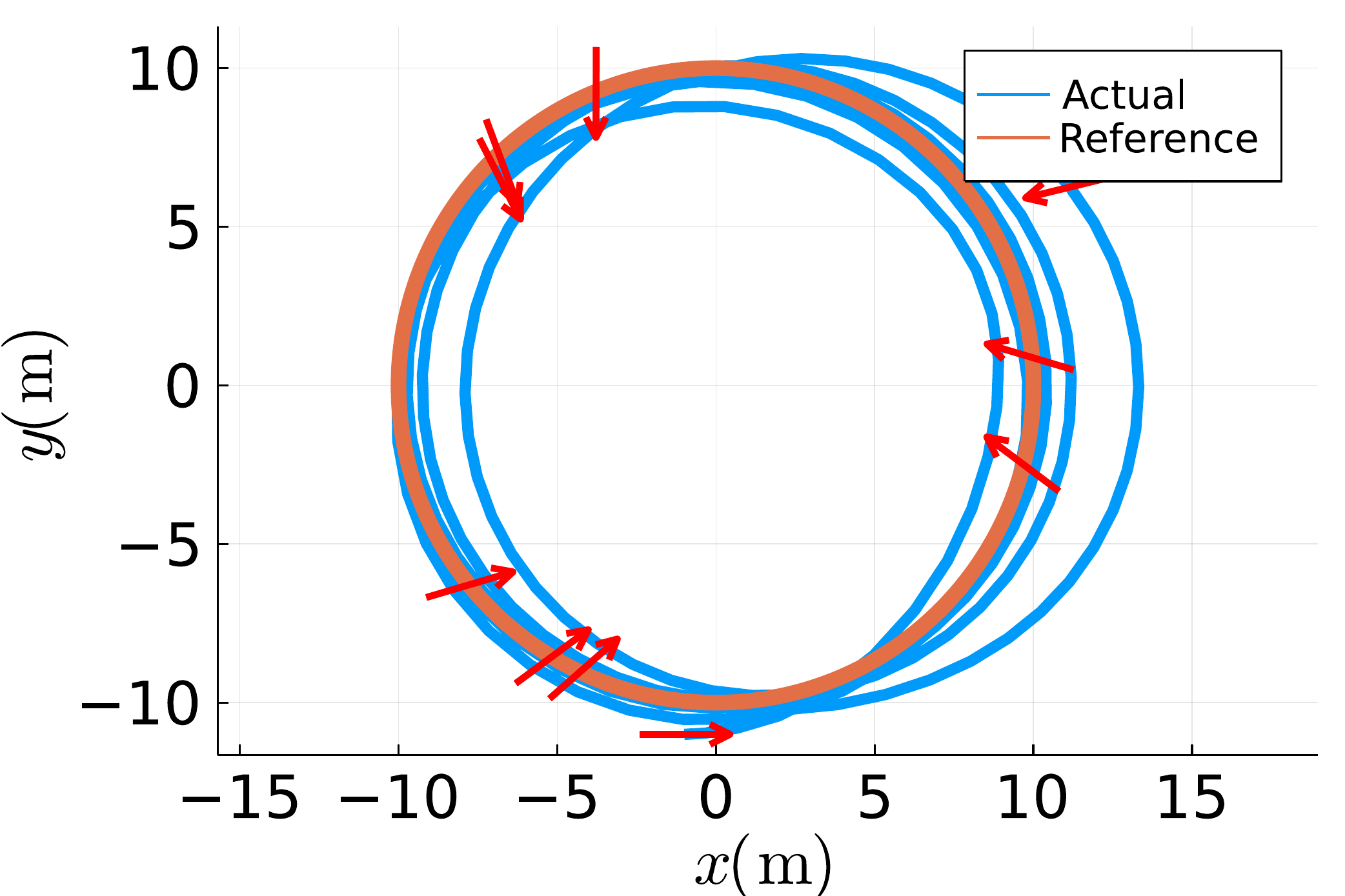}
        \caption{Reference and actual trajectories}\label{fig:vary-p3}
    \end{subfigure}
    \begin{subfigure}{0.23\textwidth}
        \centering
        \includegraphics[width=\textwidth]{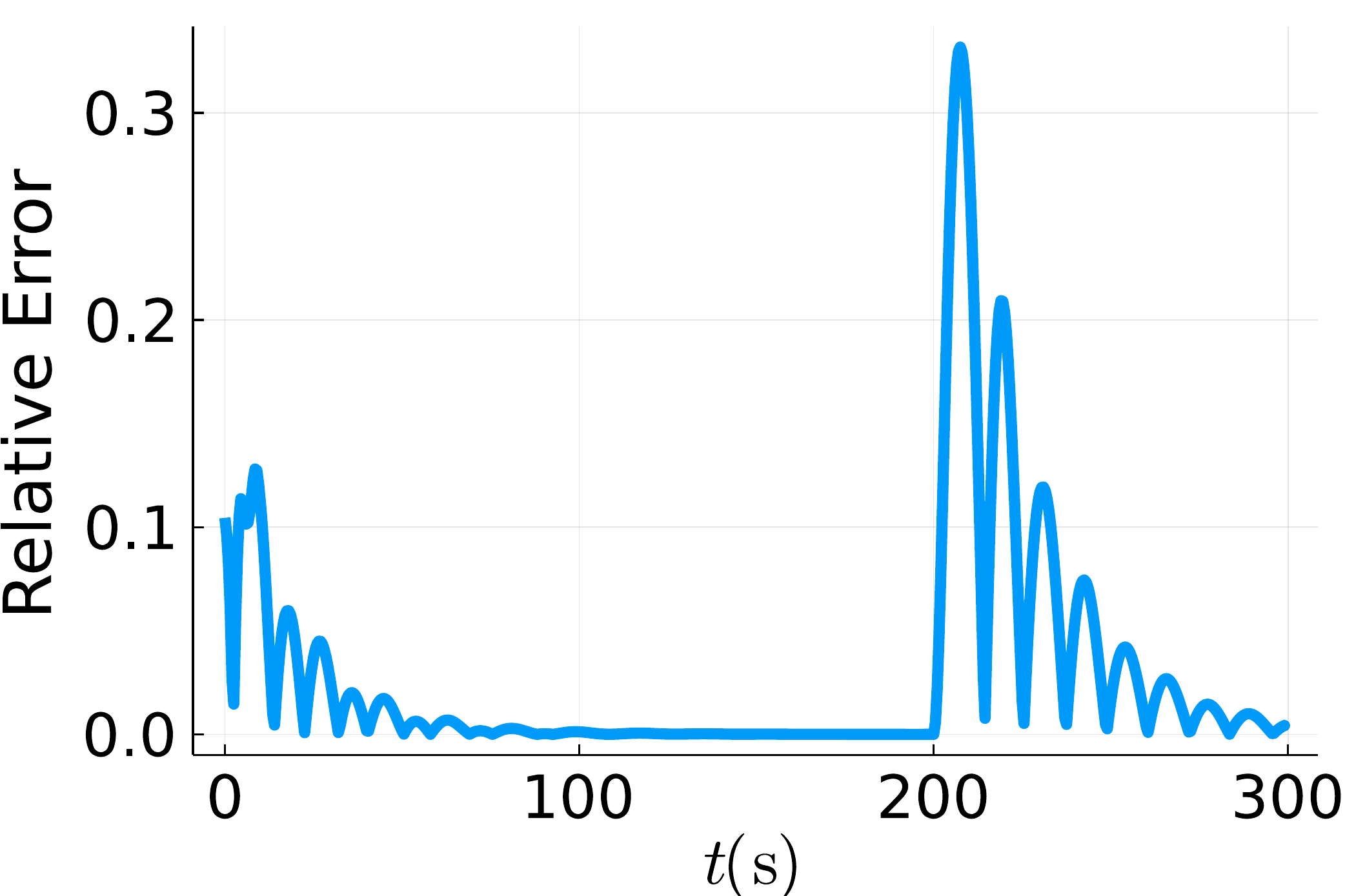}
        \caption{Relative tracking error}\label{fig:vary-p4}
    \end{subfigure}
    \caption{Simulation results for varying-interaction drifting}\label{fig:vary}
\end{figure}

\subsection{Experiment Results}

The RC car used for the hardware experiments is shown in Fig.~\ref{fig:racecar}. In order to complete the curvature estimation, we implement asynchronous Kalman filter~\cite{async_kf} to fuse the data from motion capture system and IMU. The proposed hierarchical controller is lightweight and efficient, allowing it to run on-board at a high frequency of 100Hz.
As shown in Fig.~\ref{fig:exp}, we capture 9 frames at a time interval of $1/3$ second from an experimental video and stack them together, from which we can observe that, the large sideslip angle $\beta$ is maintained while drifting.
All these drift maneuver simulations as well as full video footage of hardware experiments are available in: https://github.com/BobTesla17/hierarchical-control-framework-for-drift-maneuver.

\begin{figure}
    \centering
    \includegraphics[width=0.5\textwidth]{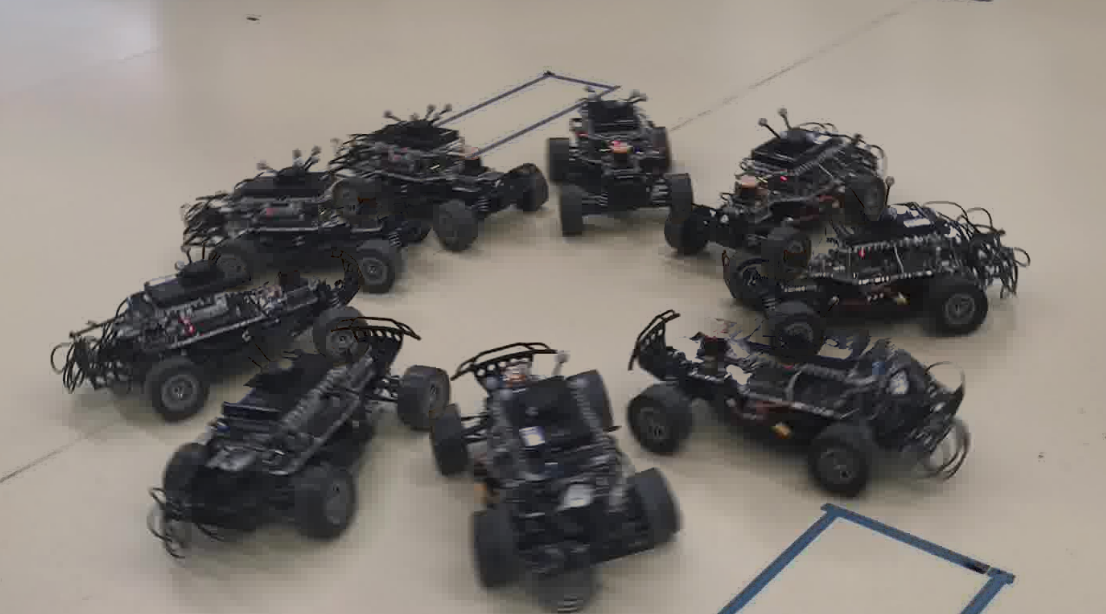}
    \caption{9 snapshots of vehicle in drifting condition stacked together}\label{fig:exp}
\end{figure}

\vspace{\spacelength}

\section{Conclusion and Future Work}\label{sec:conclusion}

In this paper, we propose a hierarchical control framework for the drift maneuvering of autonomous vehicles.
Our proposed controller features curvature inference from a small section of the historical trajectory, based on which we can decouple the curvature and center control into inner and outer control loops: the outer control loop stabilizes the center and provides a reference curvature, which is then tracked by the inner control loop.
Both numerical simulation and experimental results on a RC car platform demonstrate the effectiveness of our proposed control architecture in a set of drift maneuvering tasks.
In future works, we plan to integrate high-level curvature-based path planners into our outer-loop design, in order to achieve a larger variety of drift maneuvering tasks while avoiding obstacles in the environment.

\addtolength{\textheight}{-5cm}  % This command serves to balance the column lengths

%%%%%%%%%%%%%%%%%%%%%%%%%%%%%%%%%%%%%%%%%%%%%%%%%%%%%%%%%%%%%%%%%%%%%%%%%%%%%%%%

%%%%%%%%%%%%%%%%%%%%%%%%%%%%%%%%%%%%%%%%%%%%%%%%%%%%%%%%%%%%%%%%%%%%%%%%%%%%%%%%

%%%%%%%%%%%%%%%%%%%%%%%%%%%%%%%%%%%%%%%%%%%%%%%%%%%%%%%%%%%%%%%%%%%%%%%%%%%%%%%%

%TODO: 
%0. get an arxiv number
%1. create a github repo (add ipynb in the future)
%2. simulation of outer loop control with perfect model to see if accurate tracking can be established
%3. modelling section, start from real car                                                                  % done

% \section*{ACKNOWLEDGMENT}

\bibliography{ref.bib}

\end{document}